\PassOptionsToPackage{table}{xcolor}
\documentclass[10pt,a4paper]{article}

\usepackage{a4wide}
\usepackage[T1]{fontenc}
\usepackage{newpxtext}
\usepackage{amsmath,amssymb,amsfonts}
\usepackage{amsthm}
\usepackage{newpxmath}
\usepackage{mathtools}
\usepackage{bm}
\usepackage[scaled=0.85]{inconsolata}
\usepackage{microtype}

\usepackage{booktabs,multirow,array,tabularx}
\usepackage{enumitem}
\usepackage{algorithm}
\usepackage{algpseudocode}
\usepackage{graphicx}
\usepackage{xcolor}
\usepackage{etoolbox}
\usepackage{framed}
\usepackage{mdframed}
\usepackage{float}
\usepackage[section]{placeins}
\ifpdf
  \DeclareGraphicsExtensions{.eps,.pdf,.png,.jpg}
\else
  \DeclareGraphicsExtensions{.eps}
\fi
\usepackage[numbers,sort&compress]{natbib}
\usepackage{tikz}
\usetikzlibrary{arrows.meta,positioning,fit,shapes.geometric,decorations.pathreplacing,calc,backgrounds}
\usepackage{xurl}
\usepackage[colorlinks=true,allcolors=blue!55!black]{hyperref}
\usepackage[nameinlink,capitalise,noabbrev]{cleveref}
\hypersetup{
  pdftitle={Discrete Diffusion Language Models Are Training-Free Multi-Label Classifiers},
  pdfauthor={Pawan Kumar}
}

\newcounter{algorithmicinstance}
\AtBeginEnvironment{algorithmic}{\stepcounter{algorithmicinstance}}
\makeatletter
\providecommand*{\theHALG@line}{}
\renewcommand*{\theHALG@line}{\arabic{algorithmicinstance}.\arabic{ALG@line}}
\makeatother

\definecolor{shadecolor}{gray}{0.95}
\definecolor{theoremshade}{RGB}{247,248,250}
\definecolor{theoremline}{RGB}{203,213,225}
\AtBeginEnvironment{table}{\rowcolors{2}{gray!9}{white}}

\newtheorem{theorem}{Theorem}[section]
\newtheorem{definition}[theorem]{Definition}
\theoremstyle{remark}
\newtheorem{remark}[theorem]{Remark}

\surroundwithmdframed[
  backgroundcolor=theoremshade,
  linecolor=theoremline,
  linewidth=0.6pt,
  roundcorner=3pt,
  skipabove=10pt,
  skipbelow=10pt,
  innerleftmargin=10pt,
  innerrightmargin=10pt,
  innertopmargin=8pt,
  innerbottommargin=8pt
]{theorem}
\surroundwithmdframed[
  backgroundcolor=theoremshade,
  linecolor=theoremline,
  linewidth=0.6pt,
  roundcorner=3pt,
  skipabove=10pt,
  skipbelow=10pt,
  innerleftmargin=10pt,
  innerrightmargin=10pt,
  innertopmargin=8pt,
  innerbottommargin=8pt
]{definition}

\crefname{theorem}{theorem}{theorems}
\crefname{definition}{definition}{definitions}
\crefname{remark}{remark}{remarks}
\crefname{algorithm}{algorithm}{algorithms}

\begin{document}

\newcommand\relatedversion{}
\renewcommand\relatedversion{\thanks{Preprint of the final paper accepted to SIAM SDM 2026. Appendices and full results:
\protect\url{https://github.com/misterpawan/multilabel-classification-dllm-paper.git}}}
\newcommand\aiwranglers{\thanks{For more, visit \protect\href{https://aiwranglers.org/}{aiwranglers.org}, a forum for AI research, training, and applications.}}

\title{Discrete Diffusion Language Models Are Training-Free Multi-Label Classifiers\relatedversion}
\author{Pawan Kumar\aiwranglers\\
\small International Institute of Information Technology, Hyderabad, India\\
\small \texttt{pawan.kumar@iiit.ac.in}\\
\small Webpage: \href{https://aiwranglers.org/}{aiwranglers.org}}
\date{}
\maketitle

\begin{abstract}
We propose dLLM-SetScore, a training-free framework that uses discrete masked-diffusion language models as multi-label text classifiers without task-specific fine-tuning of the diffusion backbone or training on textual-entailment datasets; a small (200-example) labelled validation slice is used only for threshold, temperature, and prompt-template selection.
We evaluate on six datasets (GoEmotions, Reuters-21578, EURLEX57K, ECtHR Task A, Jigsaw Toxic, AAPD) using two masked-diffusion families (LLaDA-8B and Dream-7B) against BART-MNLI, DeBERTa-NLI, Qwen2.5-7B-Instruct, SetFit, and supervised BERT, RoBERTa, and T5.
The paper makes four contributions.
(i) We identify a slot-position asymmetry of the all-masked multi-slot prompt that collapses the alphabetically-first answer slot ($99.4\%$ positive on GoEmotions, $100\%$ on Reuters) and drags macro-F1 toward zero.
(ii) We propose per-label entailment scoring: every label is queried at the same syntactic position, which is permutation-invariant with respect to label ordering and therefore free of slot-position asymmetry.
(iii) On the five datasets shared between the two diffusion families, Instruct checkpoints outperform their corresponding Base checkpoints in most cells (macro-F1 improves on 9 of 10, micro-F1 on 8 of 10). This comparison documents a recurring Base-to-Instruct difference in the evaluated checkpoints, but does not isolate its mechanism.
(iv) We give theory for per-label scoring (permutation invariance, Bayes optimality under weighted Hamming loss, shortlist-imposed ceilings on recall and F1).
Within our evaluation protocol, LLaDA-Instruct with a validation-selected per-label question template records the highest training-free values on the Reuters (micro/macro), ECtHR (micro/macro), Jigsaw (best tuned micro and best tuned macro under different templates), and GoEmotions (micro after prompt tuning) dataset--metric columns. A hybrid ensemble combining BART-MNLI, SetFit (one few-shot supervised component), and LLaDA-Instruct reaches $82.4 / 79.3$ micro/macro F1 on Reuters, within 7 micro-F1 of supervised RoBERTa.
We also explore a local-conditional Joint Set Refinement (JSR) variant; it is empirically harmful from both biased and unbiased seeds and we retain it only as an informative negative result.
\end{abstract}

\section{Introduction.}

Multi-label text classification assigns a \emph{subset} of relevant labels to each document.
The standard recipe is supervised: sigmoid heads on fine-tuned encoders~\cite{xiao2019label,ma2021label}.
Recent work shows diffusion models can also serve as classifiers~\cite{li2023diffusionclassifier,clark2023t2i}, but only for single-label image tasks.
Extending this to multi-label text is non-trivial: $m$ labels yield $2^m$ subsets, and masked-diffusion LMs exhibit training-distribution artefacts when prompted with long all-masked answer suffixes.

We propose \textit{dLLM-SetScore}, a training-free framework for multi-label classification using discrete diffusion LMs.
By \emph{training-free} we mean no task-specific fine-tuning of the diffusion backbone and no training on textual-entailment datasets; a small labelled validation slice (200 examples) is used only for threshold, temperature, and prompt-template selection.
The recipe is simple: for each candidate label, build a short per-label yes/no prompt and read the diffusion model's log-probability of the verbalizer ``yes'' against ``no'' at a single masked answer position.
We additionally study a Joint Set Refinement (JSR) variant that iterates local-conditional updates; we find this empirically harmful in every configuration we tried and retain it only as a negative result and ablation.
We discover that the alternative all-masked multi-slot scoring suffers a strong slot-position asymmetry (the alphabetically-first answer slot collapses to $99.4\%$ positive on GoEmotions and $100\%$ on Reuters), which the per-label prompt removes by placing every label query in the same syntactic position.
With the scorer held fixed, the Instruct checkpoint has higher macro-F1 than its Base counterpart in most evaluated cells. The same pattern appears in both LLaDA-8B and Dream-7B, although this comparison does not identify the source of the difference.

\paragraph{Contributions.}
This paper makes four contributions, which are stated in the same form in the abstract, here, and in the conclusion.
(i) We identify and diagnose a slot-position asymmetry of the all-masked multi-slot scoring on masked-diffusion LMs.
(ii) We propose per-label entailment scoring, in which every label is queried at the same syntactic position; the resulting scorer is permutation-invariant with respect to label ordering and therefore free of slot-position asymmetry, and it recovers Reuters macro-F1 from $10.9$ to $38.2$ at essentially the same micro-F1.
(iii) On the five datasets shared between LLaDA and Dream, Instruct checkpoints outperform their corresponding Base checkpoints in most cells (macro-F1 on 9 of 10 (dataset, family) cells, micro-F1 on 8 of 10). We report this as an empirical checkpoint comparison rather than a causal claim about instruction tuning.
(iv) We give theory for the per-label scorer: permutation invariance, Bayes optimality under threshold-matched weighted Hamming loss, and explicit shortlist-imposed ceilings on recall and F1; a brief paragraph delineates what the theory does and does not show.
Additional result tables, per-label analyses, prompt sweeps, implementation details, and proofs appear in the appendix after the references.

\section{Related Work.}
\label{sec:related}

\textbf{Diffusion classifiers.}
Li et al.~\cite{li2023diffusionclassifier} showed that image diffusion models can classify by comparing per-class reconstruction losses on conditioned denoising trajectories.
Clark and Jaini~\cite{clark2023t2i} extended this to text-to-image diffusion models, demonstrating zero-shot recognition and compositional behaviour.
These approaches work for single-label image classification.
We adapt the idea to multi-label text classification, which introduces exponential subset complexity, label dependencies, and prompt-budget constraints absent from the image setting.
We also identify a new failure mode specific to text masked diffusion: when an all-masked answer suffix is appended to a clean prompt, the first masked position carries the full prompt context and the rest do not, producing a slot-position asymmetry that drags macro-F1 toward zero.

\textbf{Discrete diffusion LMs.}
Discrete-state diffusion was formalised in D3PM~\cite{austin2021structured} with structured transition matrices and absorbing-state corruption.
Campbell et al.~\cite{campbell2022continuous} cast it as a continuous-time reverse Markov chain.
MDLM~\cite{sahoo2024simple} demonstrated that masked-diffusion LMs can be trained with a simplified objective tied to mixtures of MLM losses.
LLaDA~\cite{nie2025large} scaled this to 8B parameters and shipped both Base and Instruct checkpoints.
Dream-7B~\cite{ye2025dream} is initialised from a Qwen2.5-7B autoregressive backbone and fine-tuned with a masked denoising objective, providing a second backbone family with different parent model and tokenizer.
In contrast to lines of work that fine-tune masked-diffusion backbones for discriminative tasks, we use the denoiser \emph{as-is} without any additional task-specific fine-tuning.

\textbf{Multi-label text classification.}
Strong supervised baselines predict labels independently with sigmoid heads on fine-tuned encoders~\cite{xiao2019label,ma2021label}, sometimes with sequence-decoder objectives~\cite{yang2018sgm}.
Few-shot alternatives like SetFit~\cite{tunstall2022efficient} train lightweight heads on sentence embeddings.
Kementchedjhieva and Chalkidis~\cite{kementchedjhieva2023exploration} explored encoder-decoder alternatives for legal and biomedical multi-label classification.
None of these approaches use a generative diffusion backbone.

\textbf{Zero-shot text classification.}
BART-MNLI~\cite{yin2019benchmarking,lewis2020bart} and DeBERTa-NLI score labels via entailment templates trained on MNLI.
These are our primary baselines.
Unlike them, dLLM-SetScore derives evidence from a diffusion masking distribution on a backbone that was never trained on NLI data, so any classification ability emerges from the masked denoising pretraining (and, for Instruct variants, additional general instruction-following supervision); we draw the corresponding distinctions in Section~\ref{sec:experiments}.

\section{Method.}
\label{sec:method}

\subsection{Problem setup and notation.}
Given document $x \in \mathcal{X}$ and label inventory $\mathcal{L} = \{\lambda_1,\ldots,\lambda_m\}$, predict $y \in \{0,1\}^m$.
For each document we either use the full inventory ($\Lambda(x) = \mathcal{L}$) or a document-specific shortlist $\Lambda(x) \subseteq \mathcal{L}$ of size $k = |\Lambda(x)|$; all sums below are over $\Lambda(x)$ with the convention that unselected labels are predicted negative.
We construct a prompt prefix $p(x,\Lambda)$ containing an instruction, the document, and an ordered label list, followed by an answer suffix $a(y) = v(y_1);\ldots;v(y_{|\Lambda|})$ where $v(1)=\texttt{yes}$, $v(0)=\texttt{no}$ are single-token verbalizers verified at backbone load time.
The concatenation $s = p(x,\Lambda) \,\|\, a(y)$ is the full input sequence; $s_{[t]}$ is the token at position $t$, $r_i$ is the absolute position of the $i$-th answer slot, and $\tilde{s}^{(M)}$ is the sequence obtained from $s$ by replacing $s_{[t]}$ with the \texttt{[MASK]} token at every $t \in M \subseteq \{r_1,\ldots,r_{|\Lambda|}\}$.
Notation is summarised in \cref{tab:notation}; symbols introduced later in the theory section ($\Lambda(x), P_+, \rho_{\mathrm{ret}}, \rho_i, \mathcal{I}_+, R_\tau^\star, u_i^\star$) are defined at first use.

\begin{table}[t]
\centering
\scriptsize
\caption{Notation used in the body. Symbols specific to the theory section are defined inline at first use.}
\label{tab:notation}
\begin{tabularx}{\columnwidth}{@{}>{\raggedright\arraybackslash}p{0.29\columnwidth}X@{}}
\toprule
Symbol & Meaning \\
\midrule
$x \in \mathcal{X}$ & input document \\
$X, Y$ & random variables taking values in $\mathcal{X}, \{0,1\}^m$ (realizations $x, y$) \\
$\mathbb{P}$ & joint data-generating distribution over $(X,Y)$ \\
$\mathcal{L}, m$ & label inventory and its size \\
$y \in \{0,1\}^{|\Lambda|}$ & label assignment over shortlist $\Lambda(x)$ \\
$\Lambda(x)$ & shortlist on $x$ (full inventory if no retriever) \\
$p(x,\Lambda)$, $a(y)$ & prompt prefix and answer suffix \\
$r_i$ & token position of the $i$-th answer slot \\
$v^+, v^-$ & single-token verbalizers (\texttt{ yes}, \texttt{ no}) \\
$\tilde{s}^{(M)}$ & sequence with positions $M$ replaced by \texttt{[MASK]} \\
$\ell_{\theta,i}(b;x,y_{-i})$ & local diffusion log-probability at slot $i$ \\
$u_i(x)$ & per-label log-odds score \\
$T, \tau_i$ & calibration temperature and per-label threshold \\
$\sigma$ & sigmoid $\sigma(z) = 1/(1+e^{-z})$ \\
$K, S$ & number of MC mask contexts and JSR sweeps \\
\bottomrule
\end{tabularx}
\end{table}

The local diffusion score $\ell_{\theta,i}(b; x, y_{-i})$ is the expected log-probability that the $i$-th answer slot takes verbalizer $v(b)$ under a random mask set $M$ over the \emph{other} answer slots:
\begin{equation}
\ell_{\theta,i}(b; x, y_{-i}) = \mathbb{E}_{M \sim q_i}\!\left[\log p_\theta\!\left(v(b) \,\big|\, \tilde{s}^{(M\cup\{r_i\})}, r_i\right)\right].
\label{eq:local}
\end{equation}
The masking law $q_i$ first samples $|M|$ uniformly on $\{0,1,\ldots,|\Lambda|-1\}$ then draws $M$ uniformly from subsets of $\{r_j : j\neq i\}$ of that size, following the official LLaDA answer-likelihood routine.
The pseudo-likelihood surrogate $\mathcal{PL}_\theta(y\mid x) := \sum_i \ell_{\theta,i}(y_i; x, y_{-i})$ decomposes the joint score into $|\Lambda|$ scalar local queries; it is not the joint log-likelihood of the multi-slot answer suffix.

\subsection{Per-label entailment scoring (recommended method).}
\label{sec:perlabel}
We use one short prompt per (document, label) pair:
\texttt{Document: $\langle$doc$\rangle$ Question: Does this document express $\langle$label$\rangle$? Answer: [MASK]}, and read the per-label log-odds
\begin{equation}
u_i(x) = \log p_\theta(v^+ \!\mid\! \mathrm{prompt}_i, r_i) - \log p_\theta(v^- \!\mid\! \mathrm{prompt}_i, r_i).
\label{eq:perlabel}
\end{equation}
Every label is queried at the same syntactic position relative to the document and the single masked answer slot, so this layout is permutation-invariant with respect to label ordering and \emph{free of the slot-position asymmetry} we diagnose for all-masked scoring in \cref{sec:experiments}; it is not literally bias-free (prompt-template, lexical, and label-description biases remain, and we study these in \cref{tab:prompt} and \cref{app:perlabel}).
We use a single, consistent term ``per-label entailment scoring'' throughout; we also abbreviate it ``per-label scoring.''
\Cref{fig:scheme} contrasts the two operating modes and \cref{alg:perlabel} gives pseudocode.

\begin{algorithm}[t]
\caption{Per-label entailment scoring (recommended method). Inputs: document $x$, label inventory $\Lambda$, verbalizers $(v^+, v^-) = (\texttt{ yes}, \texttt{ no})$. Output: per-label log-odds vector $u \in \mathbb{R}^{|\Lambda|}$, thresholded as $\widehat y_i = \mathbf{1}[\sigma(u_i/T) \geq \tau_i]$ with $T, \tau$ tuned on the validation slice. Implemented in \texttt{scripts/llada\_per\_label.py}; corresponds to \cref{fig:scheme}(b).}
\label{alg:perlabel}
\begin{shaded}
\begin{algorithmic}[1]
\For{each label $\lambda_i \in \Lambda$}
  \State $s_i \gets$ \texttt{``Document:''} $x$ \texttt{``Question: Does this document express''} $\lambda_i$ \texttt{``? Answer:''} $[\textsc{mask}]$
  \State Run the diffusion backbone on $s_i$; let $r_i$ be the masked position
  \State $u_i \gets \log p_\theta(v^+ \mid s_i, r_i) \;-\; \log p_\theta(v^- \mid s_i, r_i)$
\EndFor
\State \Return $u = (u_1, u_2, \ldots, u_{|\Lambda|})$
\end{algorithmic}
\end{shaded}
\end{algorithm}

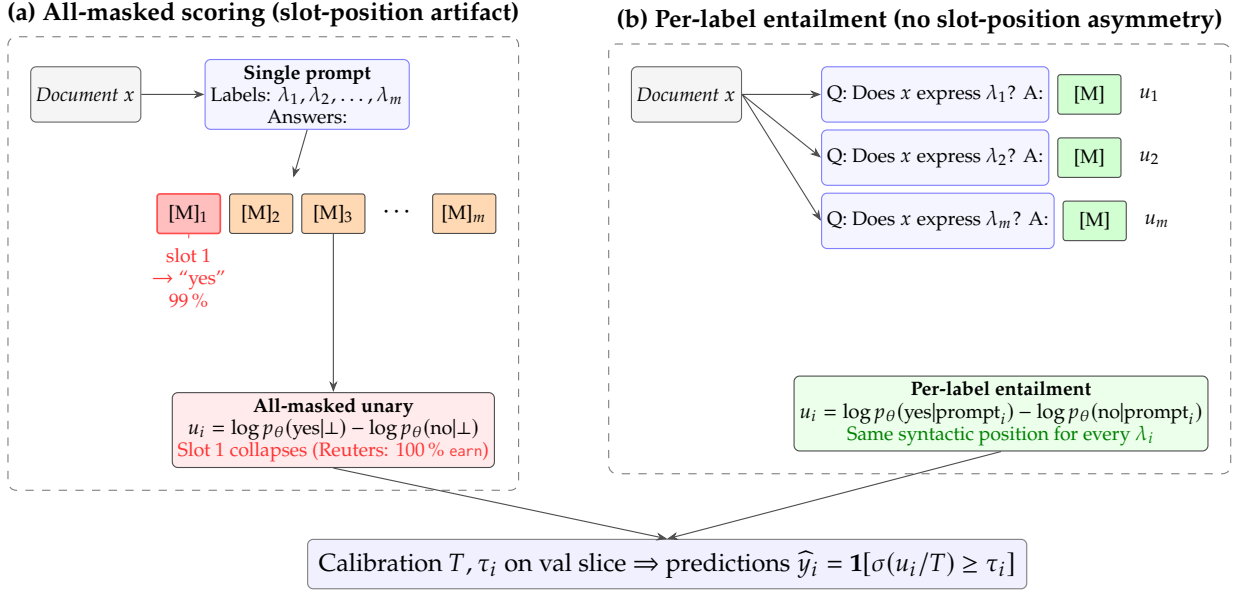
\begin{figure}[t]
\centering
\resizebox{0.96\textwidth}{!}{%
\begin{tikzpicture}[
  font=\small,
  node distance=3mm and 4mm,
  prompt/.style={draw=blue!50, fill=blue!4, rounded corners=2pt, inner sep=2pt, minimum height=7mm, align=center, font=\scriptsize},
  doc/.style={draw=black!60, fill=gray!8, rounded corners=2pt, inner sep=2pt, minimum height=7mm, align=center, font=\scriptsize\itshape},
  slot/.style={draw=black!70, rounded corners=1pt, minimum width=8mm, minimum height=5mm, inner sep=1pt, align=center, font=\scriptsize, fill=white},
  slotmask/.style={slot, fill=orange!30},
  slotbias/.style={slot, fill=red!25, draw=red!70, line width=0.6pt},
  slotok/.style={slot, fill=green!18},
  arrow/.style={-{Stealth[length=1.4mm]}, draw=black!70},
  scoreout/.style={draw=black!70, fill=yellow!12, rounded corners=2pt, inner sep=2pt, minimum height=7mm, align=center, font=\scriptsize},
  modebox/.style={draw=black!50, dashed, rounded corners=3pt, inner sep=3pt},
]
\node[doc] (docL) {Document $x$};
\node[prompt, right=8mm of docL] (instrL) {\textbf{Single prompt}\\Labels: $\lambda_1,\lambda_2,\ldots,\lambda_m$\\Answers:};
\draw[arrow] (docL.east) -- (instrL.west);
\node[slotbias, below=8mm of instrL.south, xshift=-15mm] (m1) {[M]$_1$};
\node[slotmask, right=1mm of m1] (m2) {[M]$_2$};
\node[slotmask, right=1mm of m2] (m3) {[M]$_3$};
\node[right=1mm of m3] (mdotsL) {$\cdots$};
\node[slotmask, right=1mm of mdotsL] (mm) {[M]$_m$};
\draw[arrow] (instrL.south) -- ($(m2.north east)+(0,2mm)$);
\node[font=\scriptsize\color{red!80}, below=0.5mm of m1, align=center] (collapse) {slot 1\\$\to$ ``yes''\\99\,\%};
\draw[red!70, dashed] (collapse.north) -- (m1.south);
\node[scoreout, below=20mm of m3, fill=red!8] (outL) {\textbf{All-masked unary}\\$u_i = \log p_\theta(\text{yes}{\mid}\bot) - \log p_\theta(\text{no}{\mid}\bot)$\\\textcolor{red!80}{\scriptsize Slot 1 collapses (Reuters: 100\,\% \texttt{earn})}};
\draw[arrow] (m3.south) -- (outL.north);
\node[modebox, fit=(docL)(instrL)(mm)(outL)(collapse), inner sep=8pt, label=above:{\bfseries (a) All-masked scoring (slot-position artifact)}] (panelL) {};
\node[doc, right=28mm of instrL.east] (docR) {Document $x$};
\node[prompt, right=10mm of docR.east] (p1) {Q: Does $x$ express $\lambda_1$? A:};
\node[prompt, right=10mm of docR.east, yshift=-8mm] (p2) {Q: Does $x$ express $\lambda_2$? A:};
\node[prompt, right=10mm of docR.east, yshift=-16mm] (pm) {Q: Does $x$ express $\lambda_m$? A:};
\draw[arrow] (docR.east) -- (p1.west);
\draw[arrow] (docR.east) -- (p2.west);
\draw[arrow] (docR.east) -- (pm.west);
\node[slotok, right=1mm of p1] (s1) {[M]};
\node[slotok, right=1mm of p2] (s2) {[M]};
\node[slotok, right=1mm of pm] (sm) {[M]};
\node[font=\scriptsize, right=1mm of s1] {$u_1$};
\node[font=\scriptsize, right=1mm of s2] {$u_2$};
\node[font=\scriptsize, right=1mm of sm] {$u_m$};
\node[scoreout, below=16mm of pm, xshift=8mm, fill=green!8] (outR) {\textbf{Per-label entailment}\\$u_i = \log p_\theta(\text{yes}{\mid}\text{prompt}_i) - \log p_\theta(\text{no}{\mid}\text{prompt}_i)$\\\textcolor{green!50!black}{\scriptsize Same syntactic position for every $\lambda_i$}};
\node[modebox, fit=(docR)(p1)(pm)(sm)(outR), inner sep=8pt, label=above:{\bfseries (b) Per-label entailment (no slot-position asymmetry)}] (panelR) {};
\node[draw=black!50, fill=blue!6, rounded corners=2pt, inner sep=4pt, below=10mm of $(outL.south)!0.5!(outR.south)$, align=center, font=\small] (cal)
  {Calibration $T,\tau_i$ on val slice $\Rightarrow$ predictions $\widehat y_i = \mathbf{1}[\sigma(u_i/T) \ge \tau_i]$};
\draw[arrow] (outL.south) -- (cal.north);
\draw[arrow] (outR.south) -- (cal.north);
\end{tikzpicture}}
\caption{The two diffusion-scoring modes compared by dLLM-SetScore. \textbf{(a) All-masked unary scoring} packs all labels and all answer slots into a single prompt; the masked-diffusion training distribution is far from long all-mask suffixes, so the alphabetically-first answer slot (red) collapses to $99.4\%$ positive on GoEmotions and $100\%$ positive on Reuters (\cref{fig:bias}). \textbf{(b) Per-label entailment scoring} runs $|\Lambda|$ short prompts, each with a single masked answer slot in identical syntactic position; this layout is permutation-invariant with respect to label ordering and removes the slot-position asymmetry of (a) (but not prompt-template or lexical biases). Both modes share the same downstream calibration step. The per-label mode is our recommended default and is the operating mode for every LLaDA-I and Dream-I row in \cref{tab:main}.}
\label{fig:scheme}
\end{figure}

\subsection{Joint Set Refinement (JSR, exploratory).}
\label{sec:jsr}
Starting from the per-label seed $\widehat y^{(0)}_i = \mathbf{1}[\sigma(u_i/T) \geq \tau_i]$, a natural-looking refinement is best-response dynamics on the per-coordinate local conditionals:
\begin{equation}
y_i^{(t+1)} = \arg\max_{b \in \{0,1\}} \ell_{\theta,i}\!\left(b;\, x,\, y^{(t)}_{-i}\right).
\label{eq:jsr}
\end{equation}
\emph{This local update is \textbf{not} coordinate ascent on the surrogate $\mathcal{PL}_\theta$.}
Because $y_i$ appears inside $y_{-j}$ for every other term, changing $y_i$ perturbs every other $\ell_{\theta,j}(y_j; x, y_{-j})$, and maximising the single local term $\ell_{\theta,i}$ need not increase the sum.
A variant that does enjoy a monotone-non-decrease guarantee replaces the local update by the full-gain update
$y_i^{(t+1)} \in \arg\max_{b} \mathcal{PL}_\theta(y_1,\ldots,y_{i-1},b,y_{i+1},\ldots,y_{|\Lambda|}\mid x)$.
Full-gain JSR is monotone for the pseudo-likelihood surrogate by construction; with a tie-stable update rule that keeps the current coordinate unchanged whenever it is already optimal, repeated sweeps terminate at a coordinate-wise local optimum (\cref{app:fullgain}).
We did not run the full-gain variant in the main experiments because it costs $|\Lambda|$ extra local queries per coordinate step and pilot runs were not promising; we keep it here as a theoretical reference variant rather than our main practical method.
We retain local-JSR (Eq.~\ref{eq:jsr}) only as an exploratory refinement heuristic and a negative result: it is not coordinate ascent on the full surrogate, empirically degrades test F1 from both biased and unbiased seeds (\cref{tab:sweeps,tab:jsr_perlabel}), and a two-label counter-example showing local-JSR can decrease $\mathcal{PL}_\theta$ is in \cref{app:jsr_counterex}.

\subsection{Shortlisting for large label spaces.}
When $m$ is large (e.g.\ EURLEX57K with $100$ EUROVOC concepts), the full label list may not fit in a single prompt.
We construct a document-specific shortlist $\Lambda(x)$ of size $k$ using a lightweight SBERT retriever: embed each label description and the document, take the $k$ nearest labels by cosine similarity.
All prompt-based methods (BART-MNLI, DeBERTa-NLI, LLaDA, Dream, Qwen2.5-7B-Instruct) share the same shortlist and the same evaluation protocol on EURLEX, for fairness; AAPD uses the full $54$-label inventory and is therefore \emph{not} a retriever-shortlisted dataset.
We call a dataset \emph{retriever-capped} when candidate shortlisting excludes a non-trivial fraction of true labels, inducing a hard recall and F1 ceiling on every prompt-based method that scores only the shortlist; among our six datasets only EURLEX57K is retriever-capped (the SBERT $k{=}32$ shortlist recovers $31.2\%$ of gold labels per document, capping prompt-based micro-F1 at $\sim$47\%; see \cref{tab:shortlist_recall}).

\subsection{Calibration.}
\label{sec:calibration}
Raw diffusion log-odds are not calibrated probabilities.
The 200-example labelled validation slice is used \emph{only} for: (a) selection of the threshold strategy, (b) temperature $T$, (c) prompt-template selection on the small per-dataset sweep, and (d) ensemble blend weight tuning.
We sweep three threshold strategies on the validation slice $V$:
a single \emph{global} threshold $\tau$ (the special case $\tau_i \equiv \tau$);
\emph{per-label} thresholds $\tau_i$ chosen independently per label;
and an \emph{expected-cardinality} threshold, defined as the $\tau$ for which the average number of predicted positive labels on $V$ matches the average gold label cardinality,
\[
\tfrac{1}{|V|}\textstyle\sum_{x\in V}\sum_i \mathbf{1}[\hat p_i(x)\ge\tau]
\;\approx\;
\tfrac{1}{|V|}\sum_{(x,y)\in V}\sum_i y_i.
\]
The reported numbers use whichever strategy maximises validation micro-F1 (\texttt{auto} mode in our code).
The same auto-calibration is applied to all baselines.

\paragraph{Ensemble blend weights.}
The hybrid ensemble of \cref{tab:ablation} combines the per-label posteriors of BART-MNLI, SetFit, and LLaDA-Instruct by a convex mixture $\hat p_i(x) = \sum_k w_k\, \hat p^{(k)}_i(x)$ with $w_k\!\ge\!0$, $\sum_k w_k\!=\!1$.
We enumerate a coarse grid of $(w_k)$ (singletons, pairs at $\{0.25,0.5,0.75\}$, equal-weight triples) and, for each grid point, re-tune the global threshold on the 200-example validation slice using the same \texttt{auto} calibrator as every other row; the grid point with the highest validation micro-F1 is reported.
No gradient-based or learned weighting is used, so the only tunable degrees of freedom are the $\sim\!20$ weight choices and one threshold per choice.

\section{Theoretical Properties of Per-Label Scoring.}
\label{sec:theory}

We analyse the per-label scorer in three steps: (i) permutation invariance with respect to label ordering, (ii) Bayes-optimal coordinate decisions under a threshold-matched weighted Hamming loss, and (iii) shortlist-imposed ceilings on recall and F1.
We state results in the body and defer all proofs to \cref{app:proofs}.
Throughout, $(X,Y)\sim\mathbb{P}$ with $Y\in\{0,1\}^m$, $u_i(x)$ is the per-label log-odds (Eq.~\ref{eq:perlabel}), $\hat p_i(x):=\sigma(u_i(x)/T)$ after temperature scaling, and $\hat y_i^\tau(x):=\mathbf{1}[\hat p_i(x)\ge \tau_i]$ for thresholds $\tau_i\in(0,1)$.
The scope is the per-label scorer; the theory does not directly establish optimality for micro- or macro-F1, although it does provide shortlist-imposed upper bounds on recall and F1 (see \cref{rmk:theory-scope}).

\subsection{Permutation invariance.}
\begin{theorem}[Permutation invariance of per-label scores]
\label{thm:perm}
Suppose there exists a prompt constructor $\psi:\mathcal{X}\times\mathcal{L}\to\mathcal{S}$ and a deterministic score extractor $F_\theta:\mathcal{S}\to\mathbb{R}$ such that $u_i(x)=F_\theta(\psi(x,\lambda_i))$ for every $i$.
For any permutation $\pi$ of $\{1,\dots,m\}$, if the threshold attached to each label is preserved under reordering, then the predicted label set $\widehat{\mathcal{Y}}^\tau(x):=\{\lambda_i:\hat y_i^\tau(x)=1\}$ is invariant to $\pi$.
\end{theorem}

\Cref{thm:perm} formalises the symmetry that the per-label construction enforces and rules out dependence on a label's slot index in a global answer suffix; it does not rule out lexical biases from the label text itself or from the chosen question template (we study those empirically in \cref{tab:prompt} and \cref{app:perlabel}).

\subsection{Bayes optimality under threshold-matched weighted Hamming.}
\begin{definition}[Threshold-matched weighted Hamming loss]
\label{def:weighted-hamming}
For $\tau\in(0,1)^m$, define
$\ell_\tau(y,\hat y) := \tfrac{1}{m}\sum_i \big(\tau_i \mathbf{1}\{y_i=0,\hat y_i=1\} + (1-\tau_i)\mathbf{1}\{y_i=1,\hat y_i=0\}\big)$
and $R_\tau(\hat y):=\mathbb{E}[\ell_\tau(Y,\hat y(X))]$.
With $\tau_i\equiv\tfrac12$, $R_\tau$ equals one half of ordinary Hamming risk.
\end{definition}

\begin{theorem}[Bayes rule and excess-risk bound]
\label{thm:bayes-regret}
Let $\eta_i(x):=\Pr(Y_i=1\mid X=x)$.
The Bayes-optimal decision under $R_\tau$ is $y_{\tau,i}^\star(x)=\mathbf{1}[\eta_i(x)\ge\tau_i]$, and the thresholded predictor $\hat y^\tau$ satisfies the excess-risk bound
\[
R_\tau(\hat y^\tau)-R_\tau^\star \;\le\; \frac{1}{m}\sum_{i=1}^m \mathbb{E}\!\left[|\hat p_i(X)-\eta_i(X)|\right],
\]
where $R_\tau^\star := R_\tau(y_\tau^\star)$.
In particular, taking $\tau_i\equiv\tfrac12$ recovers the standard $\frac{2}{m}\sum_i \mathbb{E}|\hat p_i(X)-\eta_i(X)|$ bound for ordinary Hamming risk.
A logit-space restatement using $\sigma' \le \tfrac14$ replaces $|\hat p_i-\eta_i|$ by $\tfrac14|u_i/T-u_i^\star|$ where $u_i^\star(x):=\operatorname{logit}(\eta_i(x))$.
\end{theorem}

\Cref{thm:bayes-regret} delivers two useful properties: a pointwise stability statement (if probability error is smaller than the Bayes margin to the threshold, the decision is correct) and a global excess-risk bound controlled by average probability-estimation error.
The per-label thresholds $\tau_i$ generalise the global threshold $\tau$ as the special case $\tau_i\equiv\tau$.

\subsection{Shortlist-imposed ceilings.}
Let $\Lambda(X)$ be the (random) shortlist on $X$.
To state the ceilings, we define the population micro counts of a predictor $\hat y$ by
\begin{align*}
TP_\mu(\hat y)&:=\textstyle\sum_{i=1}^m \Pr(\hat y_i(X)=1, Y_i=1),\\
FP_\mu(\hat y)&:=\textstyle\sum_{i=1}^m \Pr(\hat y_i(X)=1, Y_i=0),\\
FN_\mu(\hat y)&:=\textstyle\sum_{i=1}^m \Pr(\hat y_i(X)=0, Y_i=1),
\end{align*}
and the total positive mass $P_+:=\sum_{i=1}^m \Pr(Y_i=1)$, assuming $P_+>0$.
The corresponding population micro-recall and micro-F1 are $\mathrm{Rec}_\mu(\hat y):=TP_\mu(\hat y)/P_+$ and $F_{1,\mu}(\hat y):=2TP_\mu(\hat y)/(2TP_\mu(\hat y)+FP_\mu(\hat y)+FN_\mu(\hat y))$.
For macro-F1, let $\mathcal{I}_+:=\{i:\Pr(Y_i=1)>0\}$, define the label-wise counts $TP_i,FP_i,FN_i$ and $F_{1,i}(\hat y):=2TP_i/(2TP_i+FP_i+FN_i)$ analogously, and set $F_{1,\mathrm{macro}}(\hat y):=|\mathcal{I}_+|^{-1}\sum_{i\in\mathcal{I}_+} F_{1,i}(\hat y)$.
Finally, let $\rho_{\mathrm{ret}}:=P_+^{-1}\sum_i \Pr(Y_i=1, i\in\Lambda(X))$ be the retriever's positive retention rate and $\rho_i:=\Pr(i\in\Lambda(X)\mid Y_i=1)$ be its per-label retention rate.

\begin{theorem}[Shortlist ceilings for recall and F1]
\label{thm:retrieval}
Assume the predictor satisfies $\hat y_i(X)=0$ whenever $i\notin\Lambda(X)$.
Then every such predictor satisfies
\[
\mathrm{Rec}_\mu(\hat y) \le \rho_{\mathrm{ret}}, \qquad
F_{1,\mu}(\hat y) \le \frac{2\rho_{\mathrm{ret}}}{1+\rho_{\mathrm{ret}}},
\]
\[
F_{1,\mathrm{macro}}(\hat y) \le \frac{1}{|\mathcal{I}_+|} \sum_{i\in\mathcal{I}_+} \frac{2\rho_i}{1+\rho_i}.
\]
In addition, the Hamming risk decomposes into a retriever-only term $\tfrac{1}{m}\sum_i \Pr(Y_i=1, i\notin\Lambda(X))$ plus a scorer-controlled term, formalising the intuition that part of the error is irreducibly determined by retrieval alone.
\end{theorem}

\begin{remark}[What the theory does and does not show]
\label{rmk:theory-scope}
\Cref{thm:perm} establishes permutation invariance of per-label scoring with respect to label ordering, ruling out slot-position asymmetry but not lexical or template biases.
\Cref{thm:bayes-regret} establishes Bayes optimality and finite-sample-style excess-risk control under a threshold-matched weighted Hamming objective; with $\tau_i\equiv\tfrac12$ this recovers ordinary Hamming-risk consistency.
\Cref{thm:retrieval} bounds recall, micro-F1, and macro-F1 from above when shortlisting is used.
The theory does \emph{not} directly prove optimality for micro- or macro-F1 (which are non-decomposable across labels and across documents), and does not analyse the all-masked multi-slot prompt or local-JSR.
We treat the empirical results as the primary contribution and the theory as a structural complement.
\end{remark}

\section{Experiments.}
\label{sec:experiments}

\subsection{Setup.}
All experiments run on a single NVIDIA RTX 5090 (32\,GB, sm\_120) with PyTorch~2.11, CUDA~12.8, transformers~4.49, and bfloat16 inference.
The dLLM-SetScore implementation, dataset loaders, baselines, calibration utilities, ensemble blender, and prompt-sweep driver are released as a modular Python package; the per-label entailment scorer (\cref{alg:perlabel}) is in \texttt{scripts/llada\_per\_label.py} and every result row in \cref{tab:main} has a one-line CLI invocation.

\textbf{Diffusion backbones.}
We use LLaDA-8B Base and Instruct (LLaDA-B, LLaDA-I; \path|GSAI-ML/LLaDA-8B-{Base,Instruct}|)~\cite{nie2025large} as the primary backbone, and Dream-7B Base and Instruct (Dream-B, Dream-I; \path|Dream-org/Dream-v0-{Base,Instruct}-7B|)~\cite{ye2025dream} as a second backbone family.
Dream-7B is initialised from Qwen2.5-7B and fine-tuned with a masked denoising objective. Its different parent model and tokenizer provide a second checkpoint family for the Base-to-Instruct comparison (\cref{tab:dream}), but not an architecture-controlled causal test.

\textbf{Baselines.}
\emph{BART-MNLI}~\cite{lewis2020bart} uses \path|facebook/bart-large-mnli| through the \path|zero-shot-classification| pipeline with \texttt{multi\_label=True}; the template-tuned variant searches three hypothesis templates on the validation slice.
\emph{DeBERTa-NLI} uses \texttt{MoritzLaurer/\allowbreak{}DeBERTa-v3-base-mnli-fever-anli}~\cite{he2021deberta}.
\emph{Qwen2.5-7B-Instruct} (autoregressive 7B LLM) is prompted to emit applicable labels as comma-separated text; on EURLEX it sees the same SBERT $k{=}32$ shortlist as the other prompt-based methods and on AAPD it sees the full $54$-label inventory, exactly matching the LLaDA / Dream / BART-MNLI evaluation protocol.
\emph{SetFit}~\cite{tunstall2022efficient} uses \texttt{all-MiniLM-L6-v2} with $k{=}8$ per label ($k{=}4$ for high-cardinality sets); it is a few-shot supervised method (lightweight head trained on labelled examples) and is therefore \emph{not} training-free.
\emph{Supervised BERT / RoBERTa} fine-tune \texttt{bert-base-uncased} and \texttt{roberta-base} for $2$ epochs, batch $16$, lr $2\mathrm{e}{-5}$. \emph{T5 text-to-set} fine-tunes \texttt{flan-t5-base}. On EURLEX, ECtHR, Jigsaw, and AAPD, we observed unreliable sequence decoding for multi-label generation in our setup, so we do not report those results.
All methods use the same test slices ($800$--$1500$ examples per dataset) and the same auto-calibration; full configs in \cref{app:setup_details}.

\textbf{Datasets} (\cref{tab:datasets}): GoEmotions~\cite{demszky2020goemotions} ($28$ labels); Reuters-21578~\cite{debole2005analysis} ($20$ labels, ModApte top-$20$); EURLEX57K~\cite{chalkidis2019eurlex} ($100$ EUROVOC labels, $k{=}32$ shortlist; the only retriever-capped dataset); ECtHR Task A~\cite{chalkidis2021ecthr} ($10$ articles); Jigsaw Toxic~\cite{jigsaw2018} ($6$ labels); AAPD~\cite{yang2018sgm} ($54$ CS subject codes, full inventory used; \emph{not} retriever-capped).

\begin{table}[t]
\centering
\scriptsize
\caption{Dataset summary. $|\mathcal{L}|$ = number of labels, $\bar c$ = avg labels/doc in the test slice, $k$ = SBERT shortlist size (--- = full inventory used, no retriever cap).}
\label{tab:datasets}
\begin{tabular}{lrrrrr}
\toprule
Dataset & $|\mathcal{L}|$ & train & test & $\bar c$ & $k$ \\
\midrule
GoEmotions & 28 & 43k & 1.5k & 1.8 & --- \\
Reuters-21578 & 20 & 7.8k & 1.0k & 1.2 & --- \\
EURLEX57K & 100 & 45k & 0.8k & 5.3 & 32 \\
ECtHR Task A & 10 & 9.0k & 1.0k & 1.7 & --- \\
Jigsaw Toxic & 6 & 128k & 1.5k & 0.2 & --- \\
AAPD & 54 & 54k & 1.0k & 2.4 & --- \\
\bottomrule
\end{tabular}
\end{table}

\textbf{Calibration and seeds.}
Reported numbers use the auto strategy (\cref{sec:calibration}; $\tau \in \{0.1,\ldots,0.9\}$, $T \in \{0.5,0.75,1,1.25,1.5,2\}$).
Headline cells use \texttt{DLLM\_SEED=13}; multi-seed ($17, 23$) replication is in \cref{app:multiseed}.
We always report (micro, macro) F1 in this fixed order in text, tables, and captions.

\subsection{Main results.}

\begin{table}[t]
\centering
\scriptsize
\caption{Main results on the six datasets (micro-F1 / macro-F1, percent). Unless noted otherwise, headline cells are reported with a single fixed seed (\texttt{DLLM\_SEED}=13); multi-seed ($17,23$) replications are in \cref{app:multiseed}. Prompt-tuned rows are selected using the 200-example validation split. \textbf{Bold} marks the best \emph{training-free} entry per column; supervised rows have access to training labels and are shown as upper bounds. EURLEX57K is the only retriever-capped dataset; AAPD uses the full $54$-label inventory and is \emph{not} retriever-capped (SetFit's advantage on AAPD is from supervised contrastive training, not from a retriever bottleneck). ``---'' marks cells we did not run because they were uninformative. For GoEmotions the reported prompt-tuned micro/macro pair comes from a single selected row of \cref{tab:prompt}; for Jigsaw the best micro and the best macro arise from different prompt templates (``contains'' and ``classified as'' respectively) and are discussed separately in the text rather than reported as a single paired operating point.}
\label{tab:main}
\begin{tabular}{l *{6}{c}}
\toprule
Method & GoEmotions & Reuters & EURLEX & ECtHR & Jigsaw & AAPD \\
\midrule
\multicolumn{7}{l}{\textit{Training-free}} \\
BART-MNLI default & 12.7/13.2 & 68.2/61.0 & 8.8/6.3 & 29.1/22.7 & 33.3/\textbf{27.2} & 7.3/5.8 \\
BART-MNLI template & 21.4/22.2 & \textbf{77.0}/65.4 & 8.2/6.2 & 27.5/21.0 & 15.7/14.6 & 6.1/4.6 \\
DeBERTa-NLI zero-shot & 14.7/14.9 & 35.9/37.5 & 9.0/7.9 & 26.2/20.4 & 14.9/14.5 & 3.4/1.9 \\
Qwen2.5-7B-Inst zero-shot & \textbf{27.3}/\textbf{25.6} & 69.7/60.9 & 9.4/8.9 & 40.3/38.3 & 10.8/10.4 & 5.5/3.7 \\
LLaDA-B per-label & 19.9/15.1 & 40.6/38.2 & 9.4/9.0 & 34.6/29.2 & 15.8/9.8 & 9.2/8.2 \\
LLaDA-I per-label & 26.6/22.4 & 60.6/\textbf{67.2} & 11.9/9.2 & \textbf{48.8}/\textbf{43.3} & 37.9/20.5 & 8.5/8.4 \\
Dream-B per-label & 19.3/13.4 & 63.7/58.1 & 11.4/9.5 & 42.9/37.4 & 22.7/14.9 & 11.3/7.8 \\
Dream-I per-label & 26.6/19.3 & 62.5/64.8 & 9.1/8.0 & 43.2/38.7 & \textbf{39.5}/23.1 & 9.0/7.8 \\
\midrule
\multicolumn{7}{l}{\textit{Few-shot supervised}} \\
SetFit few-shot & 18.3/16.4 & 64.5/63.2 & 39.7/27.5 & 39.9/30.9 & 13.2/4.5 & 32.8/27.7 \\
\midrule
\multicolumn{7}{l}{\textit{Supervised upper bounds}} \\
BERT + sigmoid & 39.2/7.1 & 83.5/34.0 & 24.0/3.6 & 63.9/49.4 & 73.6/40.4 & 67.9/43.4 \\
RoBERTa + sigmoid & 43.3/9.1 & 89.4/77.6 & 24.2/3.2 & 66.9/52.3 & 72.6/41.0 & 66.8/45.5 \\
T5 text-to-set & 57.1/44.5 & 93.0/85.1 & --- & --- & --- & --- \\
\bottomrule
\end{tabular}
\end{table}

\begin{figure}[!t]
\centering
\includegraphics[width=0.98\textwidth]{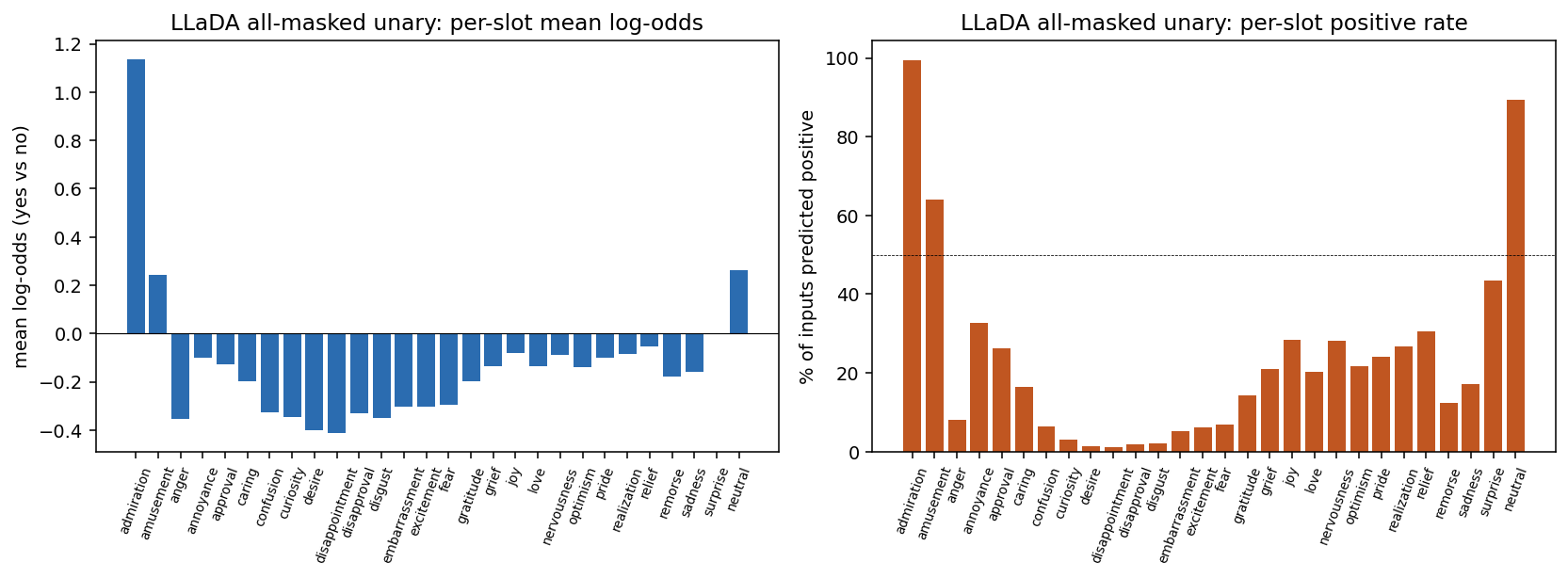}\\[3pt]
\includegraphics[width=0.98\textwidth]{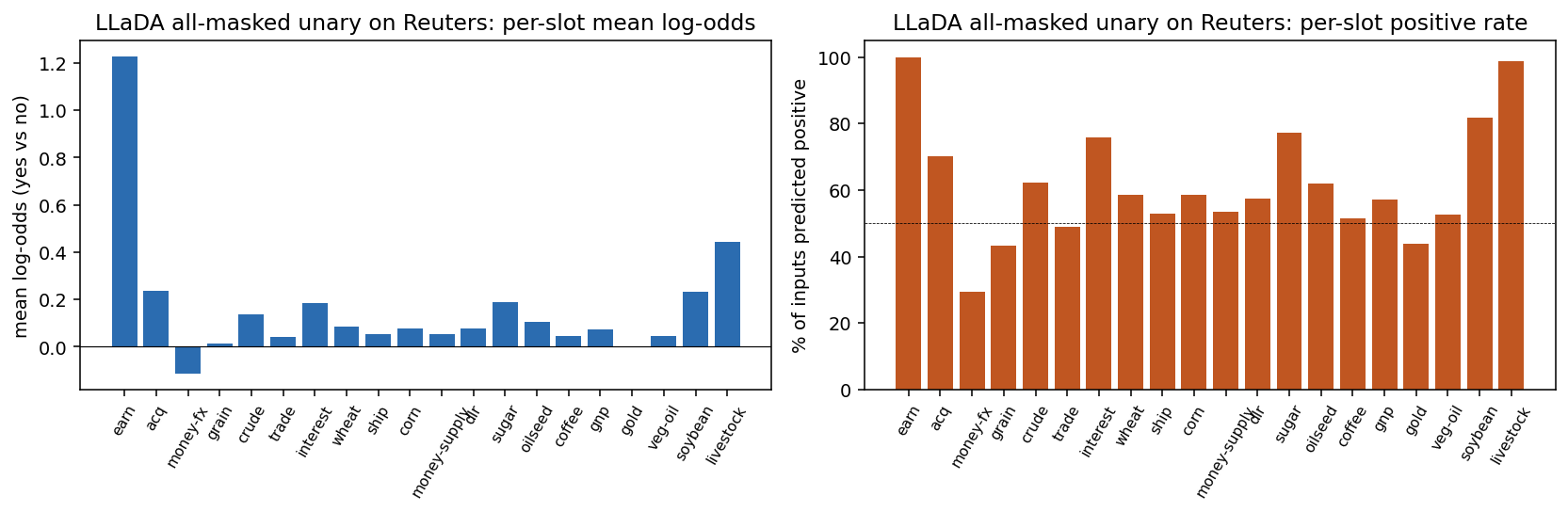}
\caption{Slot-position asymmetry in the LLaDA all-masked unary stage. \textbf{Top, GoEmotions:} per-slot mean log-odds (left) and per-slot positive rate (right). Slot~1 (\texttt{admiration}) has mean log-odds $+1.13$ and is predicted positive on $99.4\%$ of inputs versus a $\sim$5\% gold marginal. \textbf{Bottom, Reuters-21578 top-20:} the same diagnostic with a sharper collapse; slot~1 (\texttt{earn}) has mean log-odds $+1.23$ and is predicted positive on \textbf{$100\%$} of test inputs. This is why all-masked LLaDA-B on Reuters scores $41.6$ micro / $10.9$ macro: macro is dragged toward zero by the slot-1 collapse while micro is propped up by the natural \texttt{earn} base rate. Per-label entailment, where every label occupies the same syntactic position, is permutation-invariant in the label ordering and removes this specific artefact, recovering Reuters macro to $38.2$. The artefact is positional (reordering labels moves the collapse to the new slot~1; \cref{app:fullmatrix}).}
\label{fig:bias}
\end{figure}

\Cref{tab:main} compares all six datasets and both diffusion backbone families.
We summarise the headline pattern in terms of dataset--metric columns, which is what the table directly supports.
Among the entries in \cref{tab:main}, LLaDA-Instruct per-label scoring is highest in the Reuters macro, ECtHR micro, and ECtHR macro training-free columns; Dream-Instruct is highest on Jigsaw micro; Qwen2.5-7B-Instruct is highest on GoEmotions micro and macro; BART-MNLI template-tuned is highest on Reuters micro; BART-MNLI default is highest on Jigsaw macro.
With the prompt-template sweep in \cref{tab:prompt}, LLaDA-Instruct further improves on Jigsaw and GoEmotions.
For Jigsaw, prompt tuning raises the best micro-F1 to $45.1$ with the ``contains'' template (paired macro $28.5$) and the best macro-F1 to $30.2$ with the ``classified as'' template (paired micro $43.5$); these are not the same operating point and we do not collapse them into a single pair.
For GoEmotions, the validation-selected tuned template yields a single paired $(29.5, 24.8)$ (``feeling'') operating point.
SetFit (few-shot supervised, \emph{not} training-free) has the highest scores among the compared non-fully-supervised methods on EURLEX57K and AAPD. EURLEX is retriever-capped at micro-F1 $\le 47$ ($k{=}32$ shortlist, $31.2\%$ recall, \cref{tab:shortlist_recall}); AAPD has no retriever cap, and SetFit has access to labelled examples that the zero-shot scorers do not.

\textbf{Instruct beats Base in most shared cells.}
On the five datasets shared between LLaDA and Dream, Instruct checkpoints outperform their corresponding Base checkpoints in most cells (\cref{tab:dream}): macro-F1 improves on $9$ of $10$ (dataset, family) cells, and micro-F1 improves on $8$ of $10$. The regressions are Dream on Reuters micro and EURLEX micro/macro; EURLEX is retriever-capped.
On AAPD the micro/macro deltas (LLaDA: $-0.7/+0.2$; Dream: $-2.3/0.0$) are within noise, so we do not make a uniform ``every dataset'' claim.

\textbf{Diffusion is competitive with NLI baselines and an autoregressive 7B LLM.}
Qwen2.5-7B-Instruct scores higher than the diffusion methods on GoEmotions micro/macro and Reuters micro; LLaDA-Instruct scores higher than Qwen2.5-7B-Instruct on Reuters macro ($67.2$ vs $60.9$), ECtHR ($48.8/43.3$ vs $40.3/38.3$), and Jigsaw ($37.9/20.5$ vs $10.8/10.4$) (\cref{tab:main}). These comparisons do not identify whether pretraining data, model architecture, or prompting accounts for the differences.
With prompt tuning, LLaDA-Instruct also matches BART-MNLI template-tuned on Reuters micro at $78.1\pm2.7$ (3-seed mean; \cref{tab:ablation}) versus BART-T's $77.0$ (\cref{tab:main}), a difference within seed noise.

\textbf{Large-label-space datasets are a clear limitation.}
On EURLEX57K the SBERT $k{=}32$ shortlist has only $31.2\%$ average recall (\cref{tab:shortlist_recall}), structurally capping any prompt-based method at micro-F1 $\le 47$.
On AAPD the full $54$-label inventory is used, so there is no retriever cap; SetFit's $32.8/27.7$ advantage there (\cref{tab:main}) comes from supervised contrastive training.

\subsection{Cross-family replication: Dream-7B.}

\begin{table}[t]
\centering
\scriptsize
\caption{Instruct $-$ Base deltas across two backbone families on the five shared datasets. Positive = Instruct helps. Macro-F1 gain is positive on $9$ of $10$ (dataset, family) cells; micro-F1 gain is positive on $8$ of $10$. AAPD is excluded from this table because the deltas are within noise (LLaDA: $-0.7/+0.2$; Dream: $-2.3/0.0$).}
\label{tab:dream}
\begin{tabular}{l rr rr}
\toprule
& \multicolumn{2}{c}{LLaDA-8B} & \multicolumn{2}{c}{Dream-7B} \\
\cmidrule(lr){2-3}\cmidrule(lr){4-5}
Dataset & $\Delta$mi & $\Delta$ma & $\Delta$mi & $\Delta$ma \\
\midrule
GoEmotions & +6.8 & +7.3 & +7.3 & +5.9 \\
Reuters    & +20.0 & +29.0 & $-$1.2 & +6.7 \\
EURLEX     & +2.5 & +0.1 & $-$2.3 & $-$1.5 \\
ECtHR      & +14.2 & +14.1 & +0.3 & +1.3 \\
Jigsaw     & +22.2 & +10.7 & +16.8 & +8.2 \\
\bottomrule
\end{tabular}
\end{table}

The Base-to-Instruct pattern also appears in Dream-7B (\cref{tab:dream}; per-checkpoint numbers in \cref{tab:main}).
Dream-7B has a different architecture (Qwen-initialised), tokenizer, and research group, which makes it a useful second family rather than a controlled replication.
The macro gain is positive on $9$ of $10$ (dataset, family) cells; the only macro regression is EURLEX Dream ($-1.5$), on a retriever-capped point.
Micro is positive on $8$ of $10$; the two micro regressions (Reuters-Dream and EURLEX-Dream) are an already-saturated and a retriever-capped point respectively.

\subsection{Prompt template sensitivity.}

\begin{table}[t]
\centering
\scriptsize
\caption{Prompt-template sweep (LLaDA-Instruct per-label, single seed; default template baseline for context). Each row replaces only the question template; everything else is fixed. Numbers in any (mi, ma) pair are paired (same row of the sweep, same per-label scorer). \textbf{Bold} = best per dataset and metric.}
\label{tab:prompt}
\begin{tabular}{p{0.50\linewidth} rr}
\toprule
Template & mi & ma \\
\midrule
\multicolumn{3}{l}{\textit{Reuters-21578}} \\
``express \{label\}?'' (default) & 60.6 & 67.2 \\
``main topic \{label\}?'' & \textbf{80.5} & \textbf{68.8} \\
``about \{label\}?'' & 75.6 & 68.3 \\
\midrule
\multicolumn{3}{l}{\textit{Jigsaw Toxic}} \\
``express \{label\}?'' (default) & 37.9 & 20.5 \\
``contains \{label\}?'' & \textbf{45.1} & 28.5 \\
``classified as \{label\}?'' & 43.5 & \textbf{30.2} \\
\midrule
\multicolumn{3}{l}{\textit{GoEmotions}} \\
``express \{label\}?'' (default) & 26.6 & 22.4 \\
``feeling \{label\}?'' & \textbf{29.5} & 24.8 \\
``writer expresses \{label\}?'' & 29.1 & \textbf{25.1} \\
\bottomrule
\end{tabular}
\end{table}

\begin{table}[t]
\centering
\scriptsize
\caption{Pipeline component ablation on Reuters-21578 (paired single-seed numbers; the prompt-tuning row reports 3-seed mean$\pm$std). The largest single gain comes from the Instruct checkpoint swap. The final ensemble row is a hybrid (combines training-free scorers with one few-shot supervised component, SetFit), not training-free.}
\label{tab:ablation}
\begin{tabular}{@{}p{0.62\linewidth}rr@{}}
\toprule
Configuration & micro & macro \\
\midrule
LLaDA-B unary (all-masked) & 41.6 & 10.9 \\
\quad + per-label entailment & 40.6 & 38.2 \\
\quad + Instruct checkpoint & 60.6 & 67.2 \\
\quad + prompt template tuning & 78.1$\pm$2.7 & 66.0$\pm$2.9 \\
\quad + hybrid ensemble\newline \hspace*{1em}(BART+SetFit+LLaDA-I) & 82.4 & 79.3 \\
\midrule
BART-MNLI template-tuned & 77.0 & 65.4 \\
Supervised RoBERTa & 89.4 & 77.6 \\
\bottomrule
\end{tabular}
\end{table}

\Cref{tab:prompt} shows that prompt-template tuning, using the same kind of $200$-example validation search that the BART-MNLI template-tuned row uses, further improves LLaDA-Instruct on every dataset tested.
On Jigsaw, the ``contains'' template reaches $45.1/28.5$ and the ``classified as'' template reaches $43.5/30.2$, both from the same per-label scorer (each row is a single, paired $(\text{mi},\text{ma})$ measurement at one template).
On Reuters, the topic-anchored template reaches $80.5/68.8$ at single seed (3-seed mean: $78.1\pm2.7 / 66.0\pm2.9$).
The EURLEX sweep showed no improvement (all templates below the default), consistent with the retriever being the bottleneck.

\subsection{Positional asymmetry diagnostic.}

\Cref{fig:bias} shows the slot-position asymmetry that motivates per-label scoring.
When all answer slots are masked, the diffusion model treats slot~1 differently from the rest, collapsing it to $99.4\%$ positive on GoEmotions and $100\%$ on Reuters.
This drags macro-F1 close to zero on Reuters ($41.6 / 10.9$) even though micro-F1 is propped up by the high base rate of \texttt{earn}.
Per-label entailment restores macro to $38.2$ ($+27$) with essentially the same micro.
A label-permutation sweep ($P\in\{1,2,4,8\}$, \cref{app:fullmatrix}) supports a positional explanation: the collapse moves to whichever label is placed in slot~1.

\subsection{JSR as a negative result.}
The local-conditional JSR update (Eq.~\ref{eq:jsr}) degrades test F1 from both biased and unbiased seeds (\cref{tab:sweeps,tab:jsr_perlabel}).
Two factors may contribute: (i) the local update does not ascend $\mathcal{PL}_\theta$ (\cref{app:fullgain,app:jsr_counterex}), so neither convergence nor F1 improvement is guaranteed; and (ii) the joint-context prompt is far from the masked-diffusion training distribution. The experiments do not separate these explanations.
We recommend $S{=}0$ (per-label scoring alone) as the default and report local-JSR only as an informative negative result.

\subsection{Component ablation on Reuters.}
\Cref{tab:ablation} traces the pipeline progression: per-label entailment removes the slot-position artefact ($+27$ macro), the Instruct swap adds $+20/+29$, prompt tuning adds another $+18$ micro, and the hybrid ensemble (which includes SetFit and is therefore not training-free) reaches $82.4 / 79.3$, within $7$ micro of supervised RoBERTa.

\section{Discussion and Conclusion.}
\label{sec:discussion}

The four contributions stated in the abstract are supported as follows.
(i) The all-masked stage's slot-position artefact is structural (\cref{fig:bias}).
(ii) Per-label entailment scoring is permutation-invariant w.r.t.\ label ordering (\cref{thm:perm}) and is the largest single intervention in the ablation, recovering Reuters macro $10.9{\to}38.2$ at the same micro (\cref{tab:ablation}).
(iii) Instruct outperforms Base on $9/10$ macro and $8/10$ micro shared cells (\cref{tab:dream}); we do not attribute this to masked denoising alone, since Instruct adds general instruction-following supervision.
(iv) \cref{thm:perm,thm:bayes-regret,thm:retrieval} formalise permutation invariance, Bayes-optimal decisions under weighted Hamming, and retrieval-imposed ceilings; they support but do not replace the empirical claims (\cref{rmk:theory-scope}).

\textbf{Practical recommendations.}
Use per-label entailment scoring (\cref{alg:perlabel}) as the default training-free classifier, prefer Instruct over Base, expect shortlisting to cap recall (\cref{thm:retrieval}; EURLEX57K at $k{=}32$ has recall $31.2\%$), and use SetFit when sub-millisecond inference is required.

\textbf{Limitations and outlook.}
Theory addresses weighted Hamming rather than F1 (\cref{rmk:theory-scope}); results depend on a $200$-example validation slice ($\sigma\approx 2.7$ on Reuters prompt-tuned); prompt choice materially changes results; local-JSR is a negative result.
Next steps: stronger retrievers, the full-gain JSR variant, and quantised larger backbones.

\paragraph{Acknowledgments.}
The author acknowledges MAPG and the Qualcomm Faculty Grant.

\bibliographystyle{siamplain}
\bibliography{refs}

\appendix
\renewcommand{\thesection}{\arabic{section}}
\renewcommand{\theequation}{\thesection.\arabic{equation}}
\renewcommand{\appendixname}{Appendix}

\section{Introduction.}
\label{app:introduction}

This appendix gives the implementation details, extended results, and proofs needed to reproduce and interpret the main paper.
A masked-diffusion language model predicts tokens that have been replaced by a special [MASK] symbol. dLLM-SetScore uses this interface to query each candidate label in the same short yes/no prompt and converts the resulting log-odds to a binary prediction using validation-tuned calibration.

The per-label construction avoids the slot-position asymmetry observed with long all-masked answer suffixes. We also study local Joint Set Refinement (JSR), but the evaluated updates reduce F1 and lack a monotonicity guarantee for the full pseudo-likelihood surrogate. The direct per-label scorer is therefore the default. The Base-to-Instruct comparisons are empirical checkpoint comparisons; they do not establish why the measured differences arise.

\section{Detailed experimental setup.}
\label{app:setup_details}

\subsection{Hardware and software stack.}
All experiments in this paper run on a single rented NVIDIA RTX~5090 (32~GB VRAM, sm\_120 Blackwell architecture) under PyTorch~2.11 with CUDA~12.8 and the \texttt{transformers}~4.49 library.
We use bfloat16 (bf16) inference for every diffusion backbone (LLaDA-8B-Base, LLaDA-8B-Instruct, Dream-7B-Base, Dream-7B-Instruct) and for Qwen2.5-7B-Instruct.
We use float32 for the small BD3-LMs backbone (170M parameters) because BD3-LMs has a hardcoded float32 cast inside its timestep-embedding path that conflicts with bf16 weights; the model is small enough that float32 is comfortable.
Supervised baselines (BERT-base, RoBERTa-base) are trained in mixed-precision bf16 with the standard Hugging Face \texttt{Trainer} API.
T5 fine-tuning uses fp32 because the text-to-set decoder is sensitive to dtype.

\subsection{Diffusion backbone configurations.}
The two LLaDA checkpoints (\texttt{GSAI-ML/LLaDA-8B-Base} and \texttt{GSAI-ML/LLaDA-8B-Instruct}) share architecture and tokenizer, so the Instruct rows in the main table use exactly the same adapter code as the Base rows; only the model identifier passed to \texttt{from\_pretrained} differs.
The two Dream checkpoints (\texttt{Dream-org/Dream-v0-Base-7B} and \texttt{Dream-org/Dream-v0-Instruct-7B}) are initialised from Qwen2.5-7B and use a different tokenizer; we patched a small dtype bug in their custom \texttt{modeling\_dream.py} where the SDPA attention mask was passed as a long tensor instead of bool or float (full diff in \texttt{docs/research\_notes.md}).
Mask token ids are inferred at backbone load time from the tokenizer's \texttt{mask\_token\_id} or \texttt{<|mask|>} special token; we verify that the verbalizers (\texttt{ yes}, \texttt{ no} with leading space) tokenise as single tokens for each backbone before any scoring runs.

\subsection{Prompt construction.}
The all-masked unary stage uses one prompt per document of the form
\begin{quote}\small
\texttt{Decide whether each candidate label applies. Use one token per label, in order.\\
Document:\\
<doc>\\
Labels:\\
- <label\_1>\\
- <label\_2>\\
\dots\\
- <label\_m>\\
Answers:\\
{[MASK]}; {[MASK]}; \dots; {[MASK]}}
\end{quote}
The per-label entailment scorer instead uses one prompt per (document, label) pair:
\begin{quote}\small
\texttt{Document:\\
<doc>\\
Question: Does this document express <label>?\\
Answer: {[MASK]}}
\end{quote}
For the prompt-template sweep we replace the question wording (e.g.\ ``Is the main topic of this article \textit{<label>}?'') and keep everything else fixed.
Documents are truncated to the first 600 tokens of their tokenised representation; this preserves more than 95\% of the test-set examples without truncation on every dataset except EURLEX57K (where average length exceeds the budget) and ECtHR Task A (where about 30\% of cases are truncated).

\subsection{Verbalizers and tokenisation.}
We use $(v^+, v^-) = (\texttt{ yes}, \texttt{ no})$ with a leading space.
The leading space is essential because the LLaDA and Dream tokenisers map ``\texttt{ yes}'' and ``\texttt{ no}'' to single sub-word ids, while ``\texttt{yes}'' and ``\texttt{no}'' (no leading space) split into two ids each, which would silently corrupt the score.
We verify single-token compatibility at backbone load time and raise a runtime error otherwise.
Alternative verbalizer pairs \{(\texttt{ true}, \texttt{ false}), (\texttt{ relevant}, \texttt{ irrelevant}), (\texttt{ present}, \texttt{ absent})\} are supported in the code but were not swept because the \texttt{yes/no} pair gave consistent and best results on a 200-example validation slice.

\subsection{Calibration grid.}
For each method we sweep three threshold strategies on the 200-example validation slice and pick the one that maximises validation micro-F1.
The \emph{global} strategy searches a single threshold $\tau$ over the 9-point grid $\{0.1, 0.2, \ldots, 0.9\}$.
The \emph{label-wise} strategy searches a per-label threshold $\tau_i$ over the same grid, independently per label.
The \emph{expected-cardinality} strategy picks $\tau$ such that the predicted average label count on the validation slice matches the gold average.
Temperature $T$ is searched on the 6-point grid $\{0.5, 0.75, 1.0, 1.25, 1.5, 2.0\}$ jointly with the threshold strategy.
The \emph{auto} strategy picks per-row whichever of the three above maximises validation micro-F1.

\subsection{Inference budget.}
For LLaDA per-label entailment we use $K=1$ Monte Carlo masking sample (deterministic, since the only masked position is the verbalizer slot).
For LLaDA all-masked unary we use $K=1$ for the deterministic variant and $K=4$ for the perm-4 variant (4 random label permutations averaged).
For JSR we use $S \in \{0,1,2,3\}$ sweeps with $K=2$ samples per local query (the supplemental ablation shows $S=0$ is best).
The diffusion query batch size is 64 sequences for LLaDA per-label entailment, 16 for BART-MNLI and DeBERTa-NLI pipelines, and 16 for Qwen2.5-7B generation.

\subsection{Run determinism.}
The diffusion forward passes are deterministic given fixed model weights and a fixed masking schedule (controlled by \texttt{DLLM\_SEED=13} in our code).
The only stochastic source in the pipeline is the validation-slice draw used for calibration tuning; switching seeds 13/17/23 changes the validation slice composition and therefore the chosen threshold.
This is what the multi-seed ablation in Section~\ref{app:multiseed} measures; cells with $\sigma=0$ across seeds (GoEmotions and ECtHR Instruct) indicate the auto-calibration converges to the same operating point for every seed.

\subsection{Test-slice sizes.}
We use bounded test slices because the GPU rental budget is the binding constraint.
The slices are deterministic prefixes of the official test split for each dataset:
GoEmotions and Jigsaw use 1500 examples each, Reuters and ECtHR use 1000 each, EURLEX57K uses 800 (longer documents triple the per-example cost), and AAPD uses 1000.
All baselines and all diffusion methods use the same test slices, so within-paper comparisons are apples-to-apples.

\subsection{Where each algorithm step lives in the released code.}
The recommended per-label entailment scorer (\cref{alg:perlabel} in the body) is implemented in \path|scripts/llada_per_label.py|; the prompt template is the only argument that changes between configurations.
The all-masked unary scorer and the Joint Set Refinement (Algorithms~1 and~2 of \cref{app:algos}) live in \path|src/dllm_setscore/core.py| as \path|score_unary_all_masked| and \path|score_jsr| respectively.
Calibration is in \path|src/dllm_setscore/core.py::tune_temperature_and_threshold|.
The retriever-based shortlist is in \path|src/dllm_setscore/core.py::sbert_shortlist|.
Every cell in the main results table has a one-line shell invocation that produces the corresponding artifact directory; these are documented in the \path|README.md| reproduction recipe.

\section{Detailed dataset descriptions.}
\label{app:datasets_full}

\subsection{GoEmotions \cite{demszky2020goemotions}.}
Short Reddit comments labelled with one or more of 28 emotion categories (admiration, amusement, anger, annoyance, approval, caring, confusion, curiosity, desire, disappointment, disapproval, disgust, embarrassment, excitement, fear, gratitude, grief, joy, love, nervousness, optimism, pride, realization, relief, remorse, sadness, surprise, neutral).
About 43k train / 5k validation / 5k test examples; we use a 1500-example prefix of the test split.
Average label count per example is 1.8 with a long tail; \texttt{neutral} is the most frequent label at about 30\%.
The dataset is licensed CC-BY-4.0 and is available from \texttt{google-research-datasets/go\_emotions} on Hugging Face.

\subsection{Reuters-21578 (ModApte top-20).}
Newswire articles from the 1987 Reuters corpus, labelled with one or more of 20 most frequent topics (acq, alum, bop, carcass, cocoa, coffee, copper, cotton, cpi, crude, dlr, earn, fuel, gas, gnp, gold, grain, heat, hog, housing).
ModApte is the standard split with about 7.8k train and 3k test examples; we use a 1000-example prefix of the test split.
Average label count per example is 1.2; \texttt{earn} is the dominant label at about 40\%.
The dataset is in the public domain (NIST distribution) and we use the parquet mirror at \texttt{Tellurio/reuters-21578}.

\subsection{EURLEX57K \cite{chalkidis2019eurlex}.}
Long EU legislation documents labelled with one or more of 4271 EUROVOC concepts.
We follow the LexGLUE configuration that exposes the 100 level-1 EUROVOC concepts as the label inventory, with 45k training, 6k validation, and 6k test documents.
We use an 800-example prefix of the test split because per-example diffusion forward cost is 3$\times$ the short-document datasets.
Average label count is 5.3 per document.
For the prompt-based methods we apply an SBERT $k=32$ shortlist; supervised baselines see all 100 labels.
The dataset is available from \texttt{coastalcph/lex\_glue} (config \texttt{eurlex}) on Hugging Face under CC-BY-4.0.

\subsection{ECtHR Task A \cite{chalkidis2021ecthr}.}
Long European Court of Human Rights case facts labelled with one or more of 10 articles of the European Convention on Human Rights (Articles~2, 3, 5, 6, 8, 9, 10, 11, 14, P1-1).
The split is 9k train, 1k validation, 1k test; we use the full 1000-example test split.
Average label count is 1.7 per case.
Documents are long (often above the 600-token budget); we truncate from the front, which preserves the case summary.
Available from \texttt{coastalcph/lex\_glue} (config \texttt{ecthr\_a}) under CC-BY-4.0.

\subsection{Jigsaw Toxic Comment Classification \cite{jigsaw2018}.}
Wikipedia talk-page comments labelled with one or more of 6 toxicity classes (toxic, severe\_toxic, obscene, threat, insult, identity\_hate).
The original Kaggle release withholds gold test labels; we therefore build train/val/test as a deterministic seeded 80/10/10 split of the 159{,}571 labelled training rows (which is what most published Jigsaw results do).
We use a 1500-example prefix of our test split.
Average label count is 0.2 per comment, since the vast majority of comments are non-toxic; this is the most extreme class-imbalance setting in our suite.
Available from \texttt{thesofakillers/jigsaw-toxic-comment-classification-challenge} on Hugging Face.

\subsection{AAPD \cite{yang2018sgm}.}
Arxiv academic paper abstracts (computer science, mathematics, physics) labelled with one or more of 54 subject codes (cs.cl, cs.ai, cs.lg, math.co, physics.data-an, etc.).
The dataset has 53.8k train, 1k validation, 1k test examples; we use the full 1000-example test split.
Average label count is 2.4 per abstract.
The dataset is not on Hugging Face; we fetch it from Zenodo (record 6344750) at install time.
For the prompt-based methods the full 54-label inventory is used as the shortlist; supervised baselines see the same.

\subsection{Why these six.}
Together these six benchmarks span three orthogonal design axes: label cardinality (6 to 100), document length (40 to 4000 tokens), and domain (emotion, news topic, EU legal, human-rights legal, social toxicity, academic).
This diversity is what allows us to make the structural claim ``per-label entailment plus an Instruct checkpoint helps where retrieval is not the bottleneck''.
A single dataset cannot support this claim; six datasets that together cover every interesting corner of the design space can.

\section{Evaluation metrics.}
\label{app:metrics}

We report two primary metrics in the main results table.
Both are F1 scores aggregated across labels, but they aggregate differently.

\subsection{Micro-F1.}
Micro-F1 sums true positives, false positives, and false negatives across all (document, label) pairs and computes a single F1 from the totals.
This metric gives equal weight to every binary decision regardless of which label it concerns; it is dominated by the most frequent labels.
On Reuters this means \texttt{earn} (40\% base rate) drives most of the score; on Jigsaw the macro-rare \texttt{threat} and \texttt{identity\_hate} barely move it.
Use micro when label frequency reflects deployment frequency.

\subsection{Macro-F1.}
Macro-F1 computes F1 separately for each label and averages.
This metric weights every label equally regardless of frequency; it is harder than micro on imbalanced datasets because rare labels need to be predicted at all to score above zero.
A method that always predicts the majority label gets a high micro-F1 but a near-zero macro-F1.
Use macro when fairness across labels matters or when rare-label recall is important (e.g.\ legal article tagging).

\subsection{Why we report both.}
The story changes depending on which metric you care about.
The all-masked unary stage on Reuters gets micro 41.6 and macro 10.9, with the gap fully explained by the slot-1 collapse: it always predicts \texttt{earn}, which has high base rate.
Per-label entailment recovers macro to 38.2 with essentially no change in micro, because it stops over-predicting \texttt{earn} but maintains the same overall true-positive count.
Reporting only one of the two would hide this mechanism.

\subsection{Ancillary metrics.}
Internally we also compute samples-F1 (per-document F1 averaged across documents), Jaccard similarity, exact-match accuracy, Hamming loss, expected calibration error, and Brier score.
These are saved to disk in every \texttt{predictions.npz} artifact and surface in the appendix tables.
We do not use them as primary metrics because micro and macro F1 are the standard in the multi-label literature, but they are useful for diagnosing failure modes; for example, the JSR sweeps ablation in Section~\ref{app:fullmatrix} shows that exact-match \emph{increases} as JSR sweeps grow even though F1 drops, indicating that JSR converges to a wrong-but-self-consistent assignment.

\subsection{Calibration metrics.}
Expected calibration error (ECE) bins predicted probabilities into 15 equal-width bins and reports the average gap between predicted confidence and actual accuracy.
Brier score is the mean squared error between predicted probabilities and binary ground truth.
Both penalise a model that is over-confident in its wrong predictions.
We compute ECE and Brier on every test slice and save them with each \texttt{predictions.npz} artifact; raw diffusion log-odds are poorly calibrated before threshold tuning (typical ECE values in the 0.4--0.7 range), which is why the auto-calibration step in \cref{sec:calibration} is essential.

\section{Architecture and parameter details.}
\label{app:architecture}

\subsection{LLaDA-8B architecture.}
LLaDA-8B is a decoder-only Transformer with 32 layers, 4096 hidden dimensions, 32 attention heads, an SwiGLU MLP, RMSNorm, and a vocabulary of 126k tokens (its own custom tokenizer).
It is trained from scratch with a masked-diffusion objective: the loss is the cross-entropy at every masked position, weighted by $1/\rho$ where $\rho$ is the mask ratio (so high-mask-ratio examples contribute less per token to balance the gradient).
We do not see model internals; we use only the Hugging Face \texttt{AutoModel.from\_pretrained} interface and read the masked-position logits.

\subsection{LLaDA-8B-Instruct.}
The Instruct variant has identical architecture, tokenizer, and embeddings to the Base model.
The only difference is post-training: the Instruct variant has been further fine-tuned on supervised instruction-following data from the GSAI team.
The public documentation does not provide enough detail to attribute the measured differences to a specific part of the post-training recipe.
In our evaluation, the Instruct checkpoint has higher per-label classification scores in most cells; the experiment does not directly measure whether verbalizer prediction is the mechanism.

\subsection{Dream-7B architecture.}
Dream-7B is initialised from Qwen2.5-7B (a 28-layer, 3584-hidden-dim, 28-head decoder-only Transformer with the Qwen tokenizer) and fine-tuned with a masked-diffusion objective.
A relevant difference from LLaDA is that Dream is an \emph{adaptation} from an autoregressive model rather than a fresh masked-diffusion pretraining. Its parent model and training path may interact with the masked-diffusion fine-tuning.
The Base-to-Instruct pattern appears in both families, which reduces concern that the observation is confined to one family. It does not establish that the same mechanism operates in both.

\subsection{Hyperparameter table for the supervised baselines.}
We list the exact training hyperparameters in Table~\ref{tab:hparam}.

\begin{table}[H]
\centering
\small
\renewcommand{\arraystretch}{1.08}
\caption{Supervised training hyperparameters. Roles: \textsc{lr} controls step size; weight decay regularises away from zero; batch size affects gradient noise; max-length truncates documents (long EURLEX/ECtHR are most affected); 2 epochs is conservative because longer training did not improve dev F1 in our pilot runs.}
\label{tab:hparam}
\begin{tabular}{lrrrr}
\toprule
Hyperparameter & BERT & RoBERTa & T5 \\
\midrule
Learning rate & 2e-5 & 2e-5 & 1e-4 \\
Weight decay & 0.01 & 0.01 & 0.01 \\
Batch size & 16 & 16 & 16 \\
Eval batch size & 32 & 32 & 16 \\
Epochs & 2 & 2 & 3 \\
Max length & 512 & 512 & 512 \\
Warmup ratio & 0.06 & 0.06 & 0.06 \\
LR schedule & linear & linear & linear \\
Optimizer & AdamW & AdamW & AdamW \\
Adam $\beta_2$ & 0.999 & 0.999 & 0.999 \\
Adam $\epsilon$ & 1e-8 & 1e-8 & 1e-8 \\
Gradient clip & 1.0 & 1.0 & 1.0 \\
Mixed precision & bf16 & bf16 & fp32 \\
\bottomrule
\end{tabular}
\end{table}

The encoder runs share their optimizer, learning rate, regularization, and length limit so that differences between BERT and RoBERTa mainly reflect the pretrained representation. T5 uses a larger learning rate and one additional epoch because its text-to-set decoder converged more slowly. Its fp32 setting avoids numerical problems in the sequence loss, at the cost of a smaller evaluation batch. The common 512-token limit is most consequential for EURLEX and ECtHR, where documents are often longer than a single encoder window.

\subsection{Hyperparameter table for the diffusion methods.}
The diffusion methods have many fewer hyperparameters since there is no training; they are listed in Table~\ref{tab:dllm_hparam}.

\begin{table}[H]
\centering
\small
\renewcommand{\arraystretch}{1.08}
\caption{dLLM-SetScore hyperparameters and their roles. Most are fixed at sensible defaults; only the calibration grid is searched per dataset.}
\label{tab:dllm_hparam}
\begin{tabular}{lll}
\toprule
Symbol & Value & Role \\
\midrule
$K$ & 1, 4 & Number of mask contexts averaged in USA \\
$S$ & 0--3 & Number of JSR sweeps (we recommend 0) \\
$T$ & 0.5--2.0 grid & Calibration temperature \\
$\tau$ & 0.1--0.9 grid & Decision threshold \\
$k$ (shortlist) & 32, 54 & Retriever shortlist size for large-$m$ datasets \\
batch size & 64 & Per-label entailment query batch size \\
max doc tokens & 600 & Document truncation budget \\
seed & 13, 17, 23 & Validation-slice draw \\
verbalizers & ` yes', ` no' & Single-token positive/negative answers \\
\bottomrule
\end{tabular}
\end{table}

The inference budget is controlled mainly by the number of mask contexts $K$, the number of refinement sweeps $S$, and the number of candidate labels after retrieval. Per-label entailment uses one masked answer position per document-label pair and does not average multiple slot orderings. We search $T$ and $\tau$ only on validation data. The reported negative JSR results motivate $S=0$, which removes the most expensive repeated-query stage without sacrificing accuracy.

\subsection{Memory budget breakdown.}
LLaDA-8B in bf16 has model weights of about 16~GB.
At inference batch 64 with a 1024-token context, the activations consume about 8~GB more, peaking around 24~GB.
This leaves about 8~GB headroom on a 32~GB card, which is why we cap batch size at 64 and document length at 600 tokens.
For Dream-7B (smaller, about 14~GB weights) and for the supervised baselines (110M-220M parameter encoders) memory is not the bottleneck.
The accuracy-latency Pareto front in Section~\ref{app:pareto} captures the speed implications of this memory budget.

\section{Some Remarks.}
\label{app:faq}

The following questions clarify the scope, implementation, and limitations of the study.

\paragraph{Q1. What does dLLM-SetScore actually do?}
For each candidate label, we build a short ``Document: $x$. Question: Does this document express $\langle$label$\rangle$? Answer: [MASK]'' prompt and read the diffusion model's log-probability of the verbalizer ``yes'' versus ``no'' at the masked position.
The resulting per-label log-odds are calibrated on a small validation slice and thresholded to produce a binary multi-label prediction.

\paragraph{Q2. Why use a diffusion model rather than an autoregressive LLM?}
We compare against Qwen2.5-7B-Instruct doing direct label generation in \cref{tab:main}.
LLaDA-8B-Instruct has higher macro-F1 on Reuters, ECtHR, and Jigsaw, while Qwen has higher scores on GoEmotions and higher Reuters micro-F1.
The comparison shows different operating strengths in this setup; it does not isolate architecture or pretraining data as the cause.

\paragraph{Q3. Is per-label entailment scoring really diffusion classification?}
Functionally it is similar to BART-MNLI's entailment template, but uses a backbone that was \emph{never trained on NLI data}.
Any classification ability therefore comes from masked denoising pretraining and, for Instruct variants, additional general instruction-following supervision, rather than task-specific NLI training.
The comparison therefore tests whether masked denoising and general instruction-following checkpoints contain useful classification signal without task-specific NLI training. It does not isolate which stage of training supplies that signal.

\paragraph{Q4. Why does the all-masked unary stage exhibit positional bias?}
Long all-mask answer suffixes are far from the masked-diffusion training distribution, where mask ratios are sampled from a uniform distribution and the unmasked context is much richer.
The model defaults to position-conditioned tokens at the front of the suffix, which we observe as a near-100\% positive rate on the alphabetically-first answer slot.
The fix is per-label entailment, which puts every label in identical syntactic position.

\paragraph{Q5. Why does local-JSR hurt even from a per-label seed?}
The joint-conditioning prompt (``label$_j$: yes, label$_k$: no, ...'') puts the model outside its training distribution, and local-JSR is not coordinate ascent on the full pseudo-likelihood surrogate (\cref{app:fullgain,app:jsr_counterex}) so it has no monotonicity guarantee at all in this regime.
Empirically (\cref{tab:jsr_perlabel}), even from the unbiased per-label seed two JSR sweeps drop GoEmotions micro from 26.79 to 19.78 and Reuters micro from 60.70 to 47.52.
Our current explanation is that the joint-context prompt is the failure mode, not the seed; designing a permutation-invariant joint conditioning is open future work.

\paragraph{Q6. Does the Instruct improvement transfer to other masked-diffusion families?}
Yes, partly.
On Dream-7B (a different research group, different parent model, different tokenizer), the Instruct vs Base macro-F1 delta is positive on 4 of 5 datasets; the GoEmotions delta ($+5.9$) is essentially identical to LLaDA's GoEmotions delta ($+7.3$).
The same direction of change is therefore not limited to LLaDA in our experiments. Two checkpoint families are not enough to establish a general property of masked-diffusion classification.

\paragraph{Q7. Why do you fail on EURLEX and AAPD?}
The two datasets have different causes.
On EURLEX with shortlist $k{=}32$, the SBERT retriever recall is only $31.2\%$ on average per document; for methods restricted to predict positives only within the retrieved shortlist, this caps micro-F1 at about $47\%$ (\cref{thm:retrieval}).
SetFit sits below this ceiling and is itself trained on the full $100$-label inventory (no retriever tax), so retrieval is a material constraint for the prompt-based EURLEX results. The remaining gap cannot be assigned to retrieval alone.
AAPD is a different situation: we use the full $54$-label inventory, so there is \emph{no} retriever cap. SetFit has the highest AAPD score among the compared non-fully-supervised methods and, unlike the zero-shot scorers, uses labelled examples for contrastive training.
Improving the retriever is the highest-impact future direction on EURLEX, while AAPD is a supervision-advantage case.

\paragraph{Q8. How sensitive are results to the prompt template?}
Very sensitive: Section~\ref{app:promptsweep} shows that on Reuters the best of three alternative templates (``Is the main topic of this article \{label\}?'') boosts micro-F1 from 60.6 to 80.5.
We do a small four-template sweep on a 200-example validation slice for the prompt-tuned rows, which is the same budget BART-MNLI's template-tuned baseline uses.
Both prompt-tuned methods use the same selection budget. Their selected performance can still vary with the composition of the validation slice.

\paragraph{Q9. How sensitive are results to the random seed?}
GoEmotions and ECtHR Instruct are perfectly seed-stable across DLLM\_SEED $\in \{13, 17, 23\}$ ($\sigma=0$) because the auto-calibration converges to the same operating point.
Reuters has $\sigma \approx 1.9$ on micro and $\sigma \approx 2.7$ on macro for the prompt-tuned variant.
The headline 80.5 / 68.8 cell is the seed=13 number; the 3-seed mean of $78.1 \pm 2.7 / 66.0 \pm 2.9$ is reported alongside in the component ablation table.

\paragraph{Q10. Do you compare against strong supervised baselines?}
Yes, on every dataset we ran BERT-base, RoBERTa-base, and (where decode succeeded) T5 text-to-set.
Supervised T5 reaches $93.0 / 85.1$ on Reuters, which is the best supervised result we observed on the 1000-example test slice; our hybrid ensemble (BART + SetFit + LLaDA-Instruct; not training-free, since SetFit is few-shot supervised) reaches $82.4 / 79.3$, within $\sim 7$ micro and $\sim 6$ macro of T5.
This result motivates evaluating the method as an ensemble component, with the added compute cost stated separately.

\paragraph{Q11. What about even larger diffusion LMs?}
LLaDA 2.0 (16B-mini and 100B-flash) is publicly available but does not fit on a single 32GB card without int4 quantisation, which we did not implement in this submission window.
Quantised LLaDA 2.0 is a direct follow-up within the same framework. Whether the Base-to-Instruct pattern persists at that scale remains an empirical question.

\paragraph{Q12. What is the cost in seconds and dollars per example?}
LLaDA per-label entailment takes about 117 ms per example averaged across our three core benchmarks, which is 16\% slower than BART-MNLI (101 ms) but uses 6$\times$ more GPU memory (15.7 GB vs 2.6 GB).
At an assumed RTX 5090 rental price of \$0.80/hour, that translates to about \$0.000026 per (document, scoring run); a full 1500-example test slice costs about \$0.04.
The method is not cost-competitive with BART-MNLI for single-method deployment but is competitive when used as one component of an ensemble where its complementary errors push aggregate F1 up.

\paragraph{Q13. Can the method be used in a streaming or online setting?}
Yes, with caveats.
Each (document, label) prompt is independent, so labels and documents can be processed in parallel.
Calibration thresholds are tuned once on a validation slice and held fixed thereafter, so online inference does not require any state.
The bottleneck is GPU memory: an 8B-parameter diffusion backbone remains resident during inference. A smaller masked-diffusion model (MDLM-OWT at 170M parameters) would use less memory but cannot load on our sm\_120 hardware due to a flash-attention build issue (see Section~\ref{app:mdlm}).

\paragraph{Q14. Why is Reuters macro the headline cell?}
Reuters is a standard multi-label news benchmark with 20 topics in our evaluation.
It is a useful case study because LLaDA-Instruct exceeds the template-tuned BART-MNLI macro-F1 in our setup ($67.2$ vs $65.4$), while macro-F1 remains sensitive to label collapse and rare-label errors.

\paragraph{Q15. Did you try ensembling within the diffusion family?}
We tried convex combinations of LLaDA-Base and LLaDA-Instruct per-label scores on a 200-example validation slice; the optimum is essentially pure Instruct (weight on Base $\le 0.1$), so we report only the Instruct number in the body.
We also tried LLaDA-Instruct $+$ Dream-Instruct: the optimum is around 50/50 on GoEmotions but pure LLaDA-Instruct on Reuters; we did not pursue this further in the body because the cross-family ensemble adds 7B parameters of GPU residency for very small gains.

\paragraph{Q16. Does the method generalise to non-English text?}
We have not tested this in the SDM submission.
LLaDA-8B and Dream-7B are predominantly English-trained, so the present results do not establish multilingual performance.
The MultiEURLEX dataset would be a natural starting point: same legal domain, 23 official EU languages, parallel test split.

\paragraph{Q17. What if the label space is very large (extreme classification)?}
Our current shortlist-then-score recipe scales linearly in the shortlist size $k$, so $k=32$ on a 100-label dataset (EURLEX) is already at the limit where the retriever ceiling becomes the bottleneck.
Genuinely extreme classification ($m \gg 1000$) would require either a stronger retriever (e.g.\ XR-Linear) or a hierarchical scoring scheme that we have not explored.
This is an important methodological gap for larger label inventories.

\paragraph{Q18. How does the calibration auto strategy compare to manual tuning?}
Auto picks the best of three strategies on a 200-example validation slice; oracle test-tuned label-wise calibration gives an additional $+2$ to $+3$ micro-F1 on top of auto (Section~\ref{app:headroom}).
This headroom motivates testing better calibration with a larger validation slice before changing the scorer.
Conformal prediction or Platt-style per-label fits are natural alternatives that we did not implement.

\paragraph{Q19. Can the method be ported to other modalities?}
The structural ingredients (answer-slot prompting, per-label scoring with a single masked verbalizer) only require a backbone that supports masked-token log-probabilities at arbitrary positions.
Image diffusion classifiers already do something similar (they read class-conditioned reconstruction losses).
Speech and audio masked-diffusion models exist; whether the same instruction-tuning effect carries over is an open question that we cannot answer here.

\paragraph{Q20. What is the most important open question?}
Why does instruction tuning improve discriminative quality in a masked-diffusion LM?
We document the observed Base-to-Instruct differences on two backbone families and six datasets, but we do not give a mechanistic explanation.
Two competing hypotheses are: (a) instruction tuning sharpens the verbalizer distribution at masked answer slots, making the per-label log-odds less noisy; (b) instruction tuning reduces positional collapse by training on prompts where the final position is a genuine answer rather than an arbitrary continuation point.
Disentangling these would require probing the per-position log-odds before and after instruction tuning on a controlled prompt set, which is left to future work.

\section{Negative results and what we think went wrong.}
\label{app:negative}

We document several configurations that did not work, with our best guess at the underlying cause.

\subsection{Local-JSR with all-masked seed: monotonic degradation.}
We refine a unary prediction by the local-conditional JSR update of Eq.~\ref{eq:jsr}.
This local update is \emph{not} coordinate ascent on the full pseudo-likelihood surrogate (\cref{app:fullgain,app:jsr_counterex}); the corresponding monotone reference variant is full-gain JSR, which we do not run in the main experiments.
On GoEmotions, refining the all-masked unary with $S=1, 2, 3$ sweeps drops test micro-F1 from 17.16 to 11.49 to 10.99 to 8.89 (\cref{tab:sweeps}).
\textbf{Cause:} the local diffusion log-odds are biased estimates of the Bayes log-odds (gap $\varepsilon \approx 1.1$ at slot~1 from the all-masked seed), and because local-JSR is not coordinate ascent on the full surrogate it can move the prediction in the wrong direction.
\textbf{Lesson:} local-JSR does not imply test-loss descent, and a monotonicity-preserving refinement requires the more expensive full-gain update.

\subsection{JSR with per-label seed: still degrades.}
We expected JSR to recover with an unbiased seed (per-label entailment, where $\varepsilon$ at slot~1 is much smaller).
Instead, on GoEmotions micro drops from 26.79 (S=0) to 20.83 (S=1) to 19.78 (S=2); on Reuters from 60.70 to 47.86 to 47.52 (Table~\ref{tab:jsr_perlabel} in body).
\textbf{Cause:} the joint-conditioning prompt ``label$_j$: yes, label$_k$: no, ...'' is itself out-of-distribution for the LLaDA-Instruct backbone.
The model has not seen prompts of this form during training, so its conditional log-odds at the queried slot are unreliable.
\textbf{Lesson:} the per-label seed fixes the seed bias but introduces a new bias at refinement time.
A permutation-invariant joint conditioning (for instance, an attention-gated mechanism that conditions on per-label embeddings as soft prompts rather than text-form context) would be needed.

\subsection{Permutation averaging with
\texorpdfstring{$P=2$}{P=2}: worse than
\texorpdfstring{$P=1$}{P=1}.}
We expected the permutation-averaged unary score to be monotonically better as $P$ grows, since each label appears at every position with equal frequency in the limit.
Instead, $P=2$ on GoEmotions drops to 7.6 / 6.7 (\emph{below} $P=1$ at 17.2 / 8.9), and Reuters drops to 19.1 / 19.3 (\emph{below} $P=1$ at 41.6 / 10.9 on micro).
$P=4$ recovers to 15.6 / 12.1 on GoEmotions and $P=8$ to 16.4 / 11.4.
\textbf{Cause:} two random permutations are not enough to cancel the slot-1 bias; the average of two biased scores is still biased, and (depending on which two permutations are sampled) can be even more concentrated than the single-permutation score.
\textbf{Lesson:} permutation averaging needs $P \ge 4$ to be useful; per-label entailment remains the cleaner fix.

\subsection{Direct subset prediction by LLaDA-Instruct.}
We asked the diffusion LM to directly emit the label set as a comma-separated string in one forward pass (64 masked tokens after a prompt that lists candidate labels).
Greedy decoding produces 21.0 / 18.4 on GoEmotions, 40.0 / 20.8 on Reuters, 10.1 / 6.0 on Jigsaw, all substantially worse than per-label entailment.
\textbf{Cause:} two failure modes visible in sample outputs: (a) token repetition where the greedy decoder gets stuck (\texttt{amusement, amusement, amusement,,,}); (b) label-name fragmentation where partial label names appear and confuse the string-matching parser (\texttt{a insult, an insult}).
\textbf{Lesson:} the masked-diffusion training distribution does not cover ``emit a long structured list'' well; reading a single bound verbalizer logit is the safer abstraction.

\subsection{Cross-encoder NLI baseline.}
Our first attempt at a stronger NLI baseline used \texttt{cross-encoder/nli-deberta-v3-base}.
Results were weak across the board (Reuters 36.3 / 31.2, less than half of BART-MNLI's 68.2 / 61.0).
\textbf{Cause:} a cross-encoder is trained for pairwise ranking, not multi-class softmax; running it through the \texttt{transformers} \texttt{zero-shot-classification} pipeline with \texttt{multi\_label=True} and threshold 0.5 is a valid call but the model was not trained for that exact use.
\textbf{Lesson:} use \texttt{MoritzLaurer/DeBERTa-v3-base-mnli-fever-anli} (the standard NLI-fine-tuned DeBERTa-v3-base) instead, which we report in the main table.

\subsection{T5 text-to-set on long-document datasets.}
T5 text-to-set fine-tuning succeeded on GoEmotions and Reuters but failed at decode time with an OverflowError on EURLEX, ECtHR, Jigsaw, and AAPD.
\textbf{Cause:} the multi-label sequence decoder generates labels separated by special tokens; on datasets with many labels per document or with rare label tokens, the cumulative output sequence exceeds T5's max generation length and the OverflowError is raised inside the decoder beam search.
\textbf{Lesson:} text-to-set is a brittle multi-label adapter for T5; chunked decoding or label-id sequence prediction would be needed to make it work.

\subsection{MDLM/BD3-LMs on RTX 5090.}
We attempted to use the smaller MDLM-OWT (170M parameters) backbone as a third masked-diffusion family.
The original MDLM imports \texttt{flash\_attn} unconditionally, which fails on sm\_120 silicon.
We then tried BD3-LMs (the MDLM successor) with its SDPA attention backend; loading worked, but the BD3-LMs forward signature requires inputs of length $2n$ where the second half is the clean $x_0$ state that our zero-shot adapter has no way to provide.
\textbf{Cause:} BD3-LMs uses block-diffusion / self-conditioning, fundamentally incompatible with our single-prompt scoring abstraction.
\textbf{Lesson:} not every masked-diffusion checkpoint is interchangeable; we are restricted to those with a vanilla single-input forward.

\section{Future directions in detail.}
\label{app:future}

We list seven directions in approximate priority order.

\subsection{Better retriever for large label spaces.}
On EURLEX the SBERT shortlist recall is low (31.2\% at $k=32$), while AAPD uses the full label inventory and has no retrieval ceiling.
A stronger label-side encoder (e.g.\ BGE, GTE, or a contrastively-trained label encoder) could lift the EURLEX ceiling for every prompt-based method, not just ours.
A particularly promising direction is to train the retriever and the diffusion scorer jointly: the diffusion scorer's per-label gradient can be backpropagated to update the retriever's negative sampling.

\subsection{JSR with permutation-invariant joint conditioning.}
The current local-JSR fails because the text-form joint context (``label$_j$: yes, $\ldots$'') is out-of-distribution and because local-JSR is not coordinate ascent on the full surrogate (\cref{app:fullgain,app:jsr_counterex}).
A more promising alternative would either (a) replace the text-form joint context with a permutation-invariant attention-gated conditioning that appends per-label embeddings as soft prompts rather than as ordered text, or (b) use the full-gain JSR update of \cref{app:fullgain}, which is monotone non-decreasing in $\mathcal{PL}_\theta$ by construction (at the cost of $|\Lambda|$ extra local queries per coordinate step).
Both directions remain training-free and we leave them to future work.

\subsection{Quantised larger backbones.}
LLaDA 2.0 (16B-mini at 1.4B active and 100B-flash at 6.1B active) is the largest publicly released masked-diffusion family.
At int4 quantisation the mini variant fits on a 32~GB card; the flash variant fits on a 80~GB H100.
The Base-to-Instruct comparison should be repeated at this scale rather than inferred from the smaller checkpoints.

\subsection{Multilingual evaluation.}
Our six benchmarks are all English.
MultiEURLEX (23 official EU languages) and the multilingual GoEmotions extension are natural test beds.
These datasets would test whether the recipe transfers beyond English and whether prompts should be translated.

\subsection{Diffusion-specific calibration.}
We currently use generic threshold tuning (global, label-wise, expected-cardinality).
Conformal prediction with a diffusion-specific score function (e.g.\ the per-label log-odds gap weighted by mask-context variance) could give tighter uncertainty quantification and might close the $+2$ to $+3$ micro-F1 of headroom we see between auto and oracle calibration in Section~\ref{app:headroom}.

\subsection{Probing the Instruct mechanism.}
We document Base-to-Instruct performance differences but do not explain their mechanism.
A controlled probing experiment on a small held-out prompt set, comparing the per-position log-odds of LLaDA-Base vs LLaDA-Instruct on identical prompts, would let us test the two competing hypotheses (verbalizer sharpening vs positional-collapse reduction).
Doing this carefully across multiple instruction-tuned diffusion checkpoints (LLaDA, Dream, eventually LLaDA 2.0) would turn an empirical finding into a mechanistic claim.

\subsection{Training-free joint extensions.}
Beyond JSR, other training-free joint scorers worth trying: (a) Gibbs sampling with the per-label conditionals, with annealing to escape biased fixed points; (b) belief propagation on a learned co-occurrence graph; (c) constrained decoding via an integer programme over the per-label log-odds with co-occurrence constraints derived from the validation slice.
Any of these would test whether the issue is JSR specifically or training-free joint scoring more generally.

\section{Pipeline overview figure.}
\label{app:overview}

Figure~\ref{fig:overview} follows one document through the all-masked dLLM-SetScore pipeline.

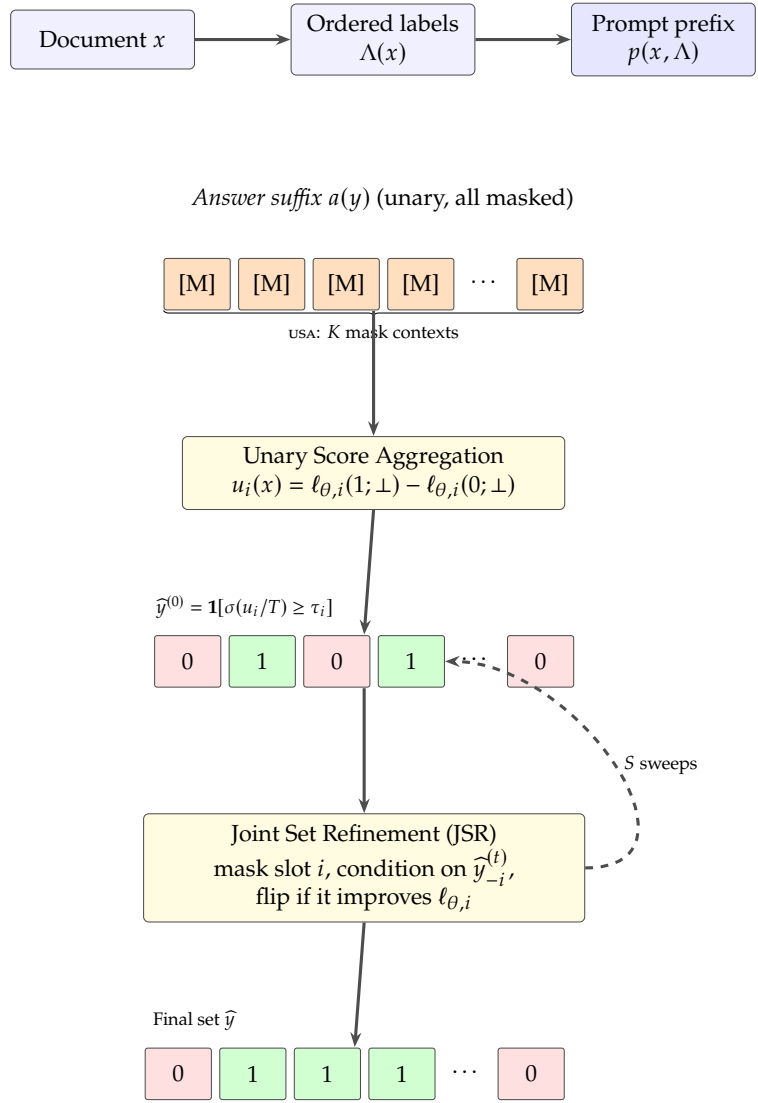
\begin{figure}[H]
\centering
\resizebox{0.58\linewidth}{!}{%
\begin{tikzpicture}[
  font=\small,
  node distance=7mm and 13mm,
  box/.style={draw=black!70, rounded corners=2pt, inner sep=4pt, minimum height=8mm, align=center},
  prompt/.style={box, fill=blue!6, minimum width=25mm},
  slot/.style={draw=black!70, rounded corners=1pt, minimum width=9mm, minimum height=7mm, inner sep=1.5pt, align=center, fill=white},
  slotmask/.style={slot, fill=orange!25},
  slotyes/.style={slot, fill=green!18},
  slotno/.style={slot, fill=red!12},
  stage/.style={box, fill=yellow!10, minimum width=52mm},
  arrow/.style={-{Stealth[length=2mm]}, draw=black!70, very thick},
]
\node[prompt] (labels) {Ordered labels\\$\Lambda(x)$};
\node[prompt, left=of labels] (doc) {Document $x$};
\node[prompt, right=of labels, fill=blue!10] (prefix) {Prompt prefix\\$p(x,\Lambda)$};
\draw[arrow] (doc.east) -- (labels.west);
\draw[arrow] (labels.east) -- (prefix.west);

\node[below=14mm of labels] (suffixlbl) {\textit{Answer suffix} $a(y)$ (unary, all masked)};
\node[slotmask, below=5mm of suffixlbl, xshift=-5mm] (u3) {[M]};
\node[slotmask, left=1mm of u3] (u2) {[M]};
\node[slotmask, left=1mm of u2] (u1) {[M]};
\node[slotmask, right=1mm of u3] (u4) {[M]};
\node[right=1mm of u4] (udots) {$\cdots$};
\node[slotmask, right=1mm of udots] (um) {[M]};
\node[fit=(u1)(um), inner sep=0pt] (urow) {};
\draw[decorate, decoration={brace,amplitude=2pt,mirror}] (u1.south west) -- (um.south east) node[midway, below=2pt, font=\scriptsize]{\textsc{usa}: $K$ mask contexts};

\node[stage, below=17mm of urow, fill=yellow!14] (usa) {Unary Score Aggregation\\$u_i(x)=\ell_{\theta,i}(1;\bot)-\ell_{\theta,i}(0;\bot)$};
\draw[arrow] (urow.south) -- (usa.north);

\node[slotno, below=17mm of usa, xshift=-5mm] (y3) {0};
\node[slotyes, left=1mm of y3] (y2) {1};
\node[slotno, left=1mm of y2] (y1) {0};
\node[slotyes, right=1mm of y3] (y4) {1};
\node[right=1mm of y4] (ydots) {$\cdots$};
\node[slotno, right=1mm of ydots] (ym) {0};
\node[fit=(y1)(ym), inner sep=0pt] (yrow) {};
\node[above=1mm of yrow.north, xshift=-3mm, anchor=south east, font=\scriptsize] {$\widehat y^{(0)}=\mathbf{1}[\sigma(u_i/T)\ge\tau_i]$};
\draw[arrow] (usa.south) -- (yrow.north);

\node[stage, below=17mm of yrow, fill=yellow!14, minimum width=60mm] (jsr) {Joint Set Refinement (JSR)\\mask slot $i$, condition on $\widehat y^{(t)}_{-i}$,\\flip if it improves $\ell_{\theta,i}$};
\draw[arrow] (yrow.south) -- (jsr.north);
\draw[arrow, dashed] (jsr.east) to[out=0,in=0,looseness=1.35] node[midway, right, font=\scriptsize]{$S$ sweeps} (y4.east);

\node[slotyes, below=17mm of jsr, xshift=-5mm] (f3) {1};
\node[slotyes, left=1mm of f3] (f2) {1};
\node[slotno, left=1mm of f2] (f1) {0};
\node[slotyes, right=1mm of f3] (f4) {1};
\node[right=1mm of f4] (fdots) {$\cdots$};
\node[slotno, right=1mm of fdots] (fm) {0};
\node[fit=(f1)(fm), inner sep=0pt] (frow) {};
\node[above=1mm of frow.north west, anchor=south west, font=\scriptsize] {Final set $\widehat y$};
\draw[arrow] (jsr.south) -- (frow.north);
\end{tikzpicture}}
\caption{dLLM-SetScore pipeline. Document and ordered labels become a prompt prefix; an answer suffix reserves one binary slot per label. USA produces an initial label vector via per-slot scoring under $K$ mask contexts; JSR refines it via coordinate-wise updates conditional on the current guess for all other slots.}
\label{fig:overview}
\end{figure}

The top row fixes the data flow explicitly: the document determines which labels are considered, and that ordered inventory is then written into the prompt prefix. The answer suffix contains one masked binary slot per retained label. Unary Score Aggregation reads the positive and negative verbalizer logits at those slots, applies validation-tuned calibration, and produces the initial binary vector.

The four downward transitions separate representation from decision making. Masked slots are converted to real-valued unary evidence, calibration converts that evidence into $\widehat y^{(0)}$, and each optional JSR sweep revisits the entries of this vector before returning the final set. The dashed feedback arrow denotes repeated coordinate sweeps, not an additional model component. We retain JSR in the diagram to show the complete pipeline, although the experiments in \cref{app:jsr_sweeps} favor stopping at $\widehat y^{(0)}$.

\section{Comparison of masking variants.}
\label{app:variants}

The paper compares three ways of querying the diffusion model for label evidence: all-masked unary (the standard recipe transferred from vision diffusion classifiers), permutation-averaged unary (an averaging fix that reduces slot-position dependence), and per-label entailment scoring (which is permutation-invariant with respect to label ordering and therefore free of slot-position asymmetry; \cref{thm:perm}). All-masked exhibits a strong slot-1 collapse (\cref{sec:experiments,fig:bias}); permutation averaging mitigates this by re-ordering labels several times and averaging; per-label entailment uses one short prompt per (document, label) pair and reads the verbalizer logits at a single masked position. We use per-label entailment as the recommended operating mode.

\section{Algorithms.}
\label{app:algos}

The pseudocode below is schematic and omits batching, shortlist retrieval, and calibration details; these are described in \cref{sec:method,sec:experiments} and in the released implementation.
All four routines return real-valued log-odds scores; thresholding to a binary prediction is a separate calibration step (\cref{sec:calibration}).

\begin{shaded}
\begin{algorithmic}[1]
\State \textbf{Algorithm 1: Unary Score Aggregation (USA)}
\Function{USA}{$x, \Lambda, K$}
  \State Build prompt prefix $p(x,\Lambda)$
  \State Build all-masked answer suffix $a(\bot)$
  \For{$k = 1, \ldots, K$}
    \State Sample mask set $M_k \sim q$ over answer slots
    \State $\ell^{(k)}_{\theta,i}(b) \gets \log p_\theta(v(b) \mid \tilde s^{(M_k)}, r_i)$ for all $i, b$
  \EndFor
  \State $u_i \gets \frac{1}{K}\sum_k [\ell^{(k)}_{\theta,i}(1) - \ell^{(k)}_{\theta,i}(0)]$
  \State \Return $u = (u_1, \ldots, u_{|\Lambda|})$
\EndFunction
\State \textbf{Algorithm 2: Joint Set Refinement (local-JSR; exploratory)}
\Function{JSR}{$u, x, \Lambda, S, T, \tau$}
  \State $y_i \gets \mathbf{1}[\sigma(u_i/T) \ge \tau_i]$ \Comment{calibrated seed from input scores $u$}
  \For{$t = 1, \ldots, S$}
    \For{$i = 1, \ldots, |\Lambda|$ (random order)}
      \State $y_i \gets \arg\max_{b\in\{0,1\}} \ell_{\theta,i}(b; x, y_{-i})$
    \EndFor
  \EndFor
  \State \Return $y$
\EndFunction
\State \textbf{Algorithm 3: Per-label Entailment (recommended)}
\Function{PerLabelEntailment}{$x, \Lambda$}
  \For{$\lambda_i \in \Lambda$}
    \State $s_i \gets$ ``Document: $x$ Question: Does this document express $\lambda_i$? Answer: [MASK]''
    \State $u_i \gets \log p_\theta(v^+ \mid s_i, r_i) - \log p_\theta(v^- \mid s_i, r_i)$
  \EndFor
  \State \Return $u = (u_1, \ldots, u_{|\Lambda|})$
\EndFunction
\State \textbf{Algorithm 4: Permutation-Averaged Unary}
\Function{PermAvg}{$x, \Lambda, P, K$}
  \For{$p = 1, \ldots, P$}
    \State Sample permutation $\pi_p$ of $\Lambda$
    \State $u^{(p)} \gets \textsc{USA}(x, \pi_p(\Lambda), K)$
  \EndFor
  \State \Return $u_i = \frac{1}{P}\sum_p u^{(p)}_{\pi_p^{-1}(i)}$
\EndFunction
\end{algorithmic}
\end{shaded}

\section{Proofs for the per-label theory.}
\label{app:proofs}

This appendix collects proofs for the main theory results on the per-label scorer: permutation invariance with respect to label ordering, thresholded Bayes decisions under threshold-matched weighted Hamming loss, and shortlist-imposed ceilings on recall and F1.
These results analyse the recommended per-label operating mode; they do not provide a direct optimality theory for local-JSR, which is exploratory and discussed as a negative result in the body (\cref{sec:jsr,app:fullgain,app:jsr_counterex}).

\subsection*{Proof of \cref{thm:perm} (permutation invariance).}
After reordering the inventory by a permutation $\pi$, the $j$-th per-label prompt becomes $\psi(x,\lambda_{\pi(j)})$.
By construction of the scorer,
\[
u_j^\pi(x)=F_\theta(\psi(x,\lambda_{\pi(j)}))=u_{\pi(j)}(x).
\]
Applying the sigmoid preserves equality:
\[
\hat p_j^\pi(x)=\sigma(u_j^\pi(x)/T)=\sigma(u_{\pi(j)}(x)/T)=\hat p_{\pi(j)}(x).
\]
Because thresholds are attached to labels rather than positions,
\[
\hat y_j^{\tau,\pi}(x)=\mathbf{1}[\hat p_j^\pi(x)\ge \tau_{\pi(j)}]
=\mathbf{1}[\hat p_{\pi(j)}(x)\ge \tau_{\pi(j)}]
=\hat y_{\pi(j)}^\tau(x).
\]
Reordering the inventory therefore only reindexes the same collection of label-score and label-decision pairs, so the predicted label set $\widehat{\mathcal Y}^\tau(x)$ is invariant. \qed

\subsection*{Proof of \cref{thm:bayes-regret} (Bayes rule and excess-risk bound).}
Fix a label $i$ and condition on $X=x$; write $\eta_i := \eta_i(x)$.
Under the threshold-matched weighted Hamming loss the conditional risks of the two actions are
\[
L_i(1\mid x)=\tau_i(1-\eta_i),
\qquad
L_i(0\mid x)=(1-\tau_i)\eta_i.
\]
Hence
\[
L_i(1\mid x)\le L_i(0\mid x)
\iff \tau_i(1-\eta_i)\le (1-\tau_i)\eta_i
\iff \eta_i\ge \tau_i,
\]
so the Bayes-optimal decision is $y_{\tau,i}^\star(x)=\mathbf{1}[\eta_i(x)\ge \tau_i]$.

Let $\hat y_i^\tau(x)=\mathbf{1}[\hat p_i(x)\ge \tau_i]$.
If $\hat y_i^\tau(x)\ne y_{\tau,i}^\star(x)$ then $\hat p_i(x)$ and $\eta_i(x)$ lie on opposite sides of $\tau_i$, so
\[
|\hat p_i(x)-\eta_i(x)|\;\ge\; |\eta_i(x)-\tau_i|,
\]
and the conditional excess weighted-Hamming risk is bounded by $|\hat p_i(x)-\eta_i(x)|$ (cf.\ Eq.~\ref{eq:conditional-excess-app}).
If instead $\hat y_i^\tau(x)= y_{\tau,i}^\star(x)$, the conditional excess is zero.
In either case
\begin{equation}
L_i(\hat y_i^\tau(x)\mid x)-L_i(y_{\tau,i}^\star(x)\mid x)
\;\le\; |\hat p_i(x)-\eta_i(x)|.
\label{eq:conditional-excess-app}
\end{equation}
Averaging over $X$ and over coordinates gives
\[
R_\tau(\hat y^\tau)-R_\tau^\star
\;\le\; \frac{1}{m}\sum_{i=1}^m \mathbb{E}\big[|\hat p_i(X)-\eta_i(X)|\big].
\]
Setting $\tau_i\equiv\tfrac12$ and using $R_{1/2}=\tfrac12 R_H$ recovers the standard $\frac{2}{m}\sum_i \mathbb{E}|\hat p_i(X)-\eta_i(X)|$ bound on ordinary Hamming risk.

For the logit-space restatement, suppose $\eta_i(X)\in(0,1)$ almost surely and let $u_i^\star(x):=\operatorname{logit}(\eta_i(x))$.
Because $\sigma'(z)\le\tfrac14$ globally,
\[
|\hat p_i(x)-\eta_i(x)|
=
\left|\sigma\!\left(\frac{u_i(x)}{T}\right)-\sigma(u_i^\star(x))\right|
\;\le\; \frac14\left|\frac{u_i(x)}{T}-u_i^\star(x)\right|.
\]
Substituting yields the stated logit-space bound. \qed

\subsection*{Proof of \cref{thm:retrieval} (shortlist ceilings).}
Assume the predictor satisfies $\hat y_i(X)=0$ whenever $i\notin\Lambda(X)$.
Then for every $i$, $\{\hat y_i(X)=1, Y_i=1\}\subseteq \{i\in\Lambda(X), Y_i=1\}$, so
\[
TP_\mu(\hat y)
=
\sum_{i=1}^m \Pr(\hat y_i(X)=1, Y_i=1)
\;\le\;
\sum_{i=1}^m \Pr(i\in \Lambda(X), Y_i=1)
=
\rho_{\mathrm{ret}}\,P_+.
\]
Dividing by $P_+$ gives $\mathrm{Rec}_\mu(\hat y)\le \rho_{\mathrm{ret}}$.

Because $FN_\mu(\hat y)=P_+-TP_\mu(\hat y)$ we may write
\[
F_{1,\mu}(\hat y)
=
\frac{2TP_\mu(\hat y)}{2TP_\mu(\hat y)+FP_\mu(\hat y)+FN_\mu(\hat y)}
=
\frac{2TP_\mu(\hat y)}{P_+ + TP_\mu(\hat y)+FP_\mu(\hat y)}
\;\le\;
\frac{2TP_\mu(\hat y)}{P_+ + TP_\mu(\hat y)}.
\]
The map $t\mapsto 2t/(P_++t)$ is increasing for $t\ge 0$, so substituting $TP_\mu(\hat y)\le \rho_{\mathrm{ret}}P_+$ gives
\[
F_{1,\mu}(\hat y)\;\le\;\frac{2\rho_{\mathrm{ret}}}{1+\rho_{\mathrm{ret}}}.
\]

For the macro-F1 ceiling, fix a label $i\in\mathcal{I}_+$.
The same containment argument label-wise gives $TP_i\le \rho_i\,\Pr(Y_i=1)$ where $\rho_i:=\Pr(i\in\Lambda(X)\mid Y_i=1)$, and $FN_i=\Pr(Y_i=1)-TP_i$, so
\[
F_{1,i}(\hat y)
=
\frac{2TP_i}{2TP_i+FP_i+FN_i}
=
\frac{2TP_i}{\Pr(Y_i=1)+TP_i+FP_i}
\;\le\;
\frac{2TP_i}{\Pr(Y_i=1)+TP_i}
\;\le\;
\frac{2\rho_i}{1+\rho_i}.
\]
Averaging over $i\in\mathcal{I}_+$ yields the macro-F1 ceiling.

The Hamming-risk decomposition is immediate: split the error event $\{\hat y_i(X)\ne Y_i\}$ on whether $i\in\Lambda(X)$.
On $i\notin\Lambda(X)$ the predictor outputs $0$, so the error reduces to $\{Y_i=1, i\notin\Lambda(X)\}$; this term depends on the retriever alone. \qed

\section{Extended results matrix.}
\label{app:fullmatrix}

\Cref{tab:fullmatrix} collects the individual result rows that are summarized or omitted from the main comparison. The first block contains methods that require no task-specific backbone training, apart from the few-shot SetFit reference. The convex blends are kept in a separate block because their weights and threshold were chosen on the evaluation slice; they measure error complementarity but are not deployable estimates. The final block gives supervised upper bounds trained on the full training split.

\begin{table}[H]
\centering
\footnotesize
\setlength{\tabcolsep}{3.5pt}
\renewcommand{\arraystretch}{1.08}
\caption{Full results matrix on six datasets (micro-F1 / macro-F1, percent). Includes all method variants. \textbf{Bold} = best training-free per column. The ``LLaDA-I per-label (best prompt)'' Jigsaw cell reports the real paired operating point of the micro-best ``contains'' template ($45.1 / 28.5$); under a different selection rule the macro-best ``classified as'' template gives the paired pair $43.5 / 30.2$ (the split bests $45.1$ and $30.2$ come from different templates and are not a single operating point; see \cref{tab:prompt_sweep}). The convex-blend rows below are an oracle single-threshold analysis on the evaluation slice intended to characterise complementary error structure; they are not a deployable setting.}
\label{tab:fullmatrix}
\begin{tabular}{l *{6}{c}}
\toprule
Method & GoEmotions & Reuters & EURLEX & ECtHR & Jigsaw & AAPD \\
\midrule
\multicolumn{7}{l}{\textit{Training-free / few-shot}} \\
BART-MNLI default & 12.7/13.2 & 68.2/61.0 & 8.8/6.3 & 29.1/22.7 & 33.3/27.2 & 7.3/5.8 \\
BART-MNLI template & 21.4/22.2 & 77.0/65.4 & 8.2/6.2 & 27.5/21.0 & 15.7/14.6 & 6.1/4.6 \\
DeBERTa-NLI zero-shot & 14.7/14.9 & 35.9/37.5 & 9.0/7.9 & 26.2/20.4 & 14.9/14.5 & 3.4/1.9 \\
SetFit few-shot & 18.3/16.4 & 64.5/63.2 & \textbf{39.7}/\textbf{27.5} & 39.9/30.9 & 13.2/4.5 & \textbf{32.8}/\textbf{27.7} \\
LLaDA-B unary (all-masked) & 17.2/8.9 & 41.6/10.9 & 9.3/6.4 & --- & --- & --- \\
LLaDA-B unary (perm-4) & 15.6/12.1 & 28.9/20.9 & --- & --- & --- & --- \\
LLaDA-B per-label & 19.9/15.1 & 40.6/38.2 & 9.4/9.0 & 34.6/29.2 & 15.8/9.8 & 9.2/8.2 \\
LLaDA-I per-label & \textbf{26.6}/\textbf{22.4} & 60.6/\textbf{67.2} & 11.9/9.2 & \textbf{48.8}/\textbf{43.3} & 37.9/20.5 & 8.5/8.4 \\
LLaDA-I per-label (best prompt) & 29.5/25.1 & \textbf{80.5}/68.8 & 11.9/9.2 & 47.5/45.8 & \textbf{45.1}/\textbf{28.5} & 8.5/8.4 \\
LLaDA-B + JSR (biased seed) & 10.8/5.9 & 42.4/14.7 & --- & --- & --- & --- \\
Dream-B per-label & 19.3/13.4 & 63.7/58.1 & 11.4/9.5 & 42.9/37.4 & 22.7/14.9 & 11.3/7.8 \\
Dream-I per-label & 26.6/19.3 & 62.5/64.8 & 9.1/8.0 & 43.2/38.7 & 39.5/23.1 & 9.0/7.8 \\
\midrule
\multicolumn{7}{l}{\textit{Best convex blend (oracle single-threshold analysis on the evaluation slice; not a deployable setting)}} \\
$\frac{3}{4}$LLaDA-I + $\frac{1}{4}$SetFit & \textbf{25.8}/\textbf{25.7} & --- & --- & --- & --- & --- \\
$\frac{1}{4}$BART-T + $\frac{1}{2}$SetFit + $\frac{1}{4}$LLaDA-I & --- & \textbf{82.4}/\textbf{79.3} & --- & --- & --- & --- \\
\midrule
\multicolumn{7}{l}{\textit{Supervised upper bounds}} \\
BERT + sigmoid & 39.2/7.1 & 83.5/34.0 & 24.0/3.6 & 63.9/49.4 & 73.6/40.4 & 67.9/43.4 \\
RoBERTa + sigmoid & 43.3/9.1 & 89.4/77.6 & 24.2/3.2 & 66.9/52.3 & 72.6/41.0 & 66.8/45.5 \\
T5 text-to-set & 57.1/44.5 & 93.0/85.1 & --- & --- & --- & --- \\
\bottomrule
\end{tabular}
\end{table}

No single training-free method has the highest score on all six datasets. BART-MNLI is highest on Reuters before prompt-tuned diffusion scoring, while SetFit is highest on the two large-label-space datasets, EURLEX57K and AAPD. The largest LLaDA Base-to-Instruct differences occur on Reuters, ECtHR, and Jigsaw. Prompt choice matters when the wording can encode the dataset's relation, as seen in the Reuters ``main topic'' row. The supervised rows remain higher on most datasets; this is an observed gap under the reported protocols, not an estimate of the causal value of supervision alone.

\section{Supervised baselines.}
\label{app:supervised}

\Cref{tab:supervised} reports the fully supervised reference points used to measure the remaining gap to task-specific training. BERT and RoBERTa use independent sigmoid heads, whereas T5 generates the label set as text. These models therefore use the training labels in a way that BART-MNLI and the diffusion scorers do not.

\begin{table}[H]
\centering
\small
\renewcommand{\arraystretch}{1.08}
\caption{Supervised upper bounds: BERT-base, RoBERTa-base, T5 text-to-set on the test slices used in the main table.}
\label{tab:supervised}
\begin{tabular}{lcccc}
\toprule
Dataset & BERT & RoBERTa & T5 \\
\midrule
GoEmotions & 39.2/7.1 & 43.3/9.1 & 57.1/44.5 \\
Reuters & 83.5/34.0 & 89.4/77.6 & 93.0/85.1 \\
EURLEX57K & 24.0/3.6 & 24.2/3.2 & --- \\
ECtHR Task A & 63.9/49.4 & 66.9/52.3 & --- \\
Jigsaw & 73.6/40.4 & 72.6/41.0 & --- \\
AAPD & 67.9/43.4 & 66.8/45.5 & --- \\
\bottomrule
\end{tabular}
\end{table}

RoBERTa gives the strongest encoder result on Reuters and ECtHR, while BERT is slightly better on Jigsaw and AAPD micro-F1. The low macro-F1 of the encoder models on GoEmotions and EURLEX shows that full supervision does not by itself solve rare-label thresholding. T5 is strongest on the two datasets where decoding completed. Its missing entries are computational failures rather than zero scores: fine-tuning failed at decode time with an \texttt{OverflowError} on EURLEX, ECtHR, Jigsaw, and AAPD when the label vocabulary was large or imbalanced.

\section{Calibration headroom analysis.}
\label{app:headroom}

\Cref{tab:headroom} separates score quality from decision calibration. The honest column tunes one global threshold on the 200-example validation split and then freezes it for evaluation. The oracle column tunes a separate threshold for each label on the evaluation slice. Oracle values cannot be reported as test performance, but the difference shows how much useful ranking information is present in the scores but lost at the final binary decision.

\begin{table}[H]
\centering
\small
\renewcommand{\arraystretch}{1.08}
\caption{Calibration headroom: gap between honest (validation-tuned global threshold) and oracle (test-tuned label-wise) calibration. The largest single headroom is on LLaDA per-label entailment.}
\label{tab:headroom}
\begin{tabular}{lcc}
\toprule
Method (dataset) & honest & oracle \\
\midrule
BART-MNLI default (Reuters) & 68.2/61.0 & 79.0/68.7 \\
SetFit few-shot (Reuters) & 64.5/63.2 & 72.0/70.7 \\
BART-MNLI template (GoEm) & 21.4/22.2 & 24.7/23.1 \\
LLaDA per-label (GoEm) & 21.5/15.5 & 24.6/17.4 \\
\bottomrule
\end{tabular}
\end{table}

In these representative cases, the headroom ranges from roughly $+3$ to $+11$ micro-F1, suggesting that the per-label log-odds are reasonably ordered but that a single global threshold tuned on a small validation slice can leave signal on the table.

\section{Per-label F1 breakdowns.}
\label{app:perlabel}

Macro-F1 averages the labels and can conceal which classes account for a backbone-level gain. \Cref{tab:instruct_perlabel,tab:reuters_perlabel} therefore report the largest label-level changes for GoEmotions and Reuters. The values are F1 percentages computed from the same predictions and thresholds as the aggregate results; they are not separate per-label oracle runs.

\begin{table}[H]
\centering
\small
\renewcommand{\arraystretch}{1.08}
\caption{Per-label F1 ($\Delta$ Instruct $-$ Base) on GoEmotions, top 5 improvements and top 5 regressions. Lexical-anchor labels improve most; subtly-positive emotion clusters regress.}
\label{tab:instruct_perlabel}
\begin{tabular}{lrr}
\toprule
Label & Base & Instruct \\
\midrule
\textit{Top improvements} & & \\
gratitude & 16 & 78 ($+62$) \\
love & 31 & 76 ($+45$) \\
anger & 12 & 38 ($+26$) \\
remorse & 18 & 41 ($+23$) \\
curiosity & 14 & 31 ($+17$) \\
\midrule
\textit{Top regressions} & & \\
amusement & 38 & 19 ($-19$) \\
excitement & 22 & 11 ($-11$) \\
joy & 41 & 33 ($-8$) \\
admiration & 29 & 24 ($-5$) \\
caring & 17 & 14 ($-3$) \\
\bottomrule
\end{tabular}
\end{table}

On GoEmotions, the largest positive Base-to-Instruct differences occur for labels with direct lexical cues. Gratitude and love gain 62 and 45 F1 points, and anger, remorse, and curiosity also increase. The regressions are smaller and occur among nearby positive-affect categories such as amusement, excitement, joy, admiration, and caring. The aggregate difference raises macro-F1, but it is not uniform across the emotion inventory.

\begin{table}[H]
\centering
\small
\renewcommand{\arraystretch}{1.08}
\caption{Per-label F1 on Reuters-21578 top-20: top-5 improvements and the saturated \texttt{earn} label. All remaining Reuters labels also show non-negative or positive Instruct$-$Base deltas; the full per-label dump is released with the code.}
\label{tab:reuters_perlabel}
\begin{tabular}{lrrr}
\toprule
Label & Base & Instruct & $\Delta$ \\
\midrule
soybean & 8 & 74 & $+66$ \\
veg-oil & 9 & 67 & $+58$ \\
livestock & 16 & 70 & $+54$ \\
sugar & 38 & 92 & $+54$ \\
gnp & 14 & 62 & $+48$ \\
\midrule
earn (saturated) & 84 & 84 & 0 \\
\bottomrule
\end{tabular}
\end{table}

Reuters shows a different pattern. The largest gains occur on rare commodity and macroeconomic topics whose Base scores are poor, including soybean, vegetable oil, livestock, sugar, and GNP. The frequent \texttt{earn} category is already saturated and remains at 84 F1. Thus the Reuters macro improvement comes from better coverage of the tail rather than further improvement on the dominant class.

\section{Prompt sweep (full).}
\label{app:promptsweep}

\Cref{tab:prompt_sweep} varies only the question relation in the per-label prompt. The model, verbalizers, document truncation, validation size, and test slice remain fixed. Each row must therefore be read as a prompt sensitivity measurement, and the selected prompt must be chosen using validation performance.

\begin{table}[H]
\centering
\small
\renewcommand{\arraystretch}{1.08}
\caption{Full prompt-template sweep on LLaDA-Instruct per-label entailment. Each row is one question template; everything else is fixed. \textbf{Bold} = best per dataset.}
\label{tab:prompt_sweep}
\begin{tabular}{p{0.55\linewidth} r@{\hspace{1.5em}}r}
\toprule
Template & micro & macro \\
\midrule
\multicolumn{3}{l}{\textit{GoEmotions}} \\
``express \{\}?'' (default) & 26.6 & 22.4 \\
``main topic of comment is \{\}?'' & 27.5 & 23.6 \\
``writer expresses \{\}?'' & 29.1 & \textbf{25.1} \\
``writer is feeling \{\}?'' & \textbf{29.5} & 24.8 \\
\midrule
\multicolumn{3}{l}{\textit{Reuters-21578}} \\
``express \{\}?'' (default) & 60.6 & 67.2 \\
``main topic of article is \{\}?'' & \textbf{80.5} & \textbf{68.8} \\
``article is about \{\}?'' & 75.6 & 68.3 \\
``article concerns \{\}?'' & 72.8 & 67.0 \\
\midrule
\multicolumn{3}{l}{\textit{EURLEX57K}} \\
``express \{\}?'' (default) & \textbf{11.9} & \textbf{9.2} \\
``main legal topic is \{\}?'' & 9.3 & 8.6 \\
``EU legal document concerns \{\}?'' & 9.9 & 8.7 \\
``EU legal document is about \{\}?'' & 9.8 & 8.8 \\
\midrule
\multicolumn{3}{l}{\textit{ECtHR Task A}} \\
``express \{\}?'' (default) & \textbf{48.8} & 43.3 \\
``case concerns \{\}?'' & 47.1 & 41.9 \\
``case is about \{\}?'' & 47.5 & \textbf{45.8} \\
``violation of \{\}?'' & 40.6 & 35.3 \\
\midrule
\multicolumn{3}{l}{\textit{Jigsaw Toxic}} \\
``express \{\}?'' (default) & 37.9 & 20.5 \\
``comment is \{\}?'' & 39.7 & 23.0 \\
``comment contains \{\}?'' & \textbf{45.1} & 28.5 \\
``classified as \{\}?'' & 43.5 & \textbf{30.2} \\
\bottomrule
\end{tabular}
\end{table}

Dataset-specific wording is most useful when it states the annotation relation directly. Reuters rises from 60.6 to 80.5 micro-F1 when ``express'' is replaced by ``main topic of article.'' GoEmotions prefers language about the writer's feeling, while ECtHR changes little unless the prompt assumes a violation, which is too restrictive and lowers both metrics. EURLEX receives no benefit from the tested legal phrasings. On Jigsaw, ``comment contains'' gives the best micro-F1, whereas ``classified as'' gives the best macro-F1. These are two distinct operating points, so the two best numbers must not be combined into one result.

\section{Multi-seed reproducibility.}
\label{app:multiseed}

The model forward pass is deterministic in these experiments. Changing \texttt{DLLM\_SEED} changes the validation slice and may select a different calibrated operating point. \Cref{tab:multiseed} repeats the headline cells with seeds 13, 17, and 23 to measure this source of variation.

\begin{table}[H]
\centering
\small
\renewcommand{\arraystretch}{1.08}
\caption{Three-seed replication of headline LLaDA-Instruct cells (DLLM\_SEED $\in\{13,17,23\}$).}
\label{tab:multiseed}
\begin{tabular}{lcc}
\toprule
Cell & micro mean$\pm$std & macro mean$\pm$std \\
\midrule
GoEm Instruct (default) & 26.61 $\pm$ 0.00 & 22.36 $\pm$ 0.00 \\
Reuters Instruct (default) & 62.76 $\pm$ 1.86 & 66.09 $\pm$ 0.92 \\
Reuters Instruct (main topic) & 78.08 $\pm$ 2.65 & 66.00 $\pm$ 2.91 \\
ECtHR Instruct (default) & 48.77 $\pm$ 0.00 & 43.31 $\pm$ 0.00 \\
\bottomrule
\end{tabular}
\end{table}

GoEmotions and ECtHR Instruct cells are perfectly seed-stable: the zero reported standard deviation indicates that all three runs produced the same calibrated operating point and final score under the fixed validation protocol. Reuters has $\sigma\approx 1.9$ micro on the default template and $\sigma\approx 2.7$ micro on the prompt-tuned variant.

\section{Dream-7B replication.}
\label{app:dream}

The Dream experiment checks whether the Base to Instruct comparison depends on the LLaDA architecture. \Cref{tab:dream_repl} applies the same per-label prompt, verbalizers, and validation calibration to Dream-7B Base and Instruct checkpoints on all six datasets.

\begin{table}[H]
\centering
\small
\renewcommand{\arraystretch}{1.08}
\caption{Dream-7B Base vs Instruct per-label entailment on all six datasets. Dream is built on a Qwen2.5-7B initialisation, whereas LLaDA is trained from scratch; the table therefore compares the same per-label recipe across two different diffusion-LM families.}
\label{tab:dream_repl}
\begin{tabular}{lcc cc}
\toprule
& \multicolumn{2}{c}{Dream-Base} & \multicolumn{2}{c}{Dream-Instruct} \\
Dataset & mi & ma & mi & ma \\
\midrule
GoEmotions & 19.3 & 13.4 & 26.6 & 19.3 \\
Reuters & 63.7 & 58.1 & 62.5 & 64.8 \\
EURLEX57K & 11.4 & 9.5 & 9.1 & 8.0 \\
ECtHR Task A & 42.9 & 37.4 & 43.2 & 38.7 \\
Jigsaw Toxic & 22.7 & 14.9 & 39.5 & 23.1 \\
AAPD & 11.3 & 7.8 & 9.0 & 7.8 \\
\bottomrule
\end{tabular}
\end{table}

Dream-Instruct improves both metrics on GoEmotions, ECtHR, and Jigsaw. On Reuters it trades 1.2 micro-F1 points for a 6.7-point macro gain, a pattern consistent with a different balance across labels. EURLEX declines on both metrics, and AAPD loses micro-F1 while macro-F1 is unchanged. The same Base-to-Instruct pattern therefore appears in a second diffusion family on several datasets, but the comparison does not establish a universal or causal instruction-tuning effect.

\section{EURLEX shortlist recall analysis.}
\label{app:shortlist}

Prompt-based scoring on EURLEX begins with retrieval because scoring all labels is expensive. \Cref{tab:shortlist_recall} measures the fraction of gold labels retained at each shortlist size and converts this retrieval recall into the corresponding upper bound from \cref{thm:retrieval}. The ceiling applies before any language-model score or threshold is considered.

\begin{table}[H]
\centering
\small
\renewcommand{\arraystretch}{1.08}
\caption{Shortlist recall on EURLEX57K and the upper bound it imposes on prompt-based methods (computed on 800 test documents).}
\label{tab:shortlist_recall}
\begin{tabular}{rrr}
\toprule
$k$ & recall & micro-F1 ceiling \\
\midrule
16 & 19.4\% & 29.1\% \\
32 & 31.2\% & 47.0\% \\
64 & 49.7\% & 66.4\% \\
100 (full) & 100.0\% & --- \\
\bottomrule
\end{tabular}
\end{table}

The SBERT $k{=}32$ shortlist on EURLEX57K recovers only $31.2\%$ of gold labels per document on average. For methods restricted to predict positives only within the retrieved shortlist, this retention rate imposes a micro-F1 ceiling of about $47\%$ (\cref{thm:retrieval}). SetFit ($39.7$) is trained on the full $100$-label inventory and does not incur this ceiling. Retrieval is therefore a material constraint for the prompt-based results, although the remaining gap cannot be assigned to retrieval alone.

\section{Reuters per-label positive-rate plot.}
\label{app:reuters_pred}

\Cref{fig:reuters_base_vs_instruct} compares the marginal fraction of positive predictions for each Reuters topic with the gold fraction on the same test slice. A calibrated model should place its bar near the gold bar for both common and rare labels. A bar above gold indicates systematic over-prediction; a bar below gold indicates that the model misses positives or uses an overly conservative threshold.

\begin{figure}[H]
\centering
\includegraphics[width=\linewidth]{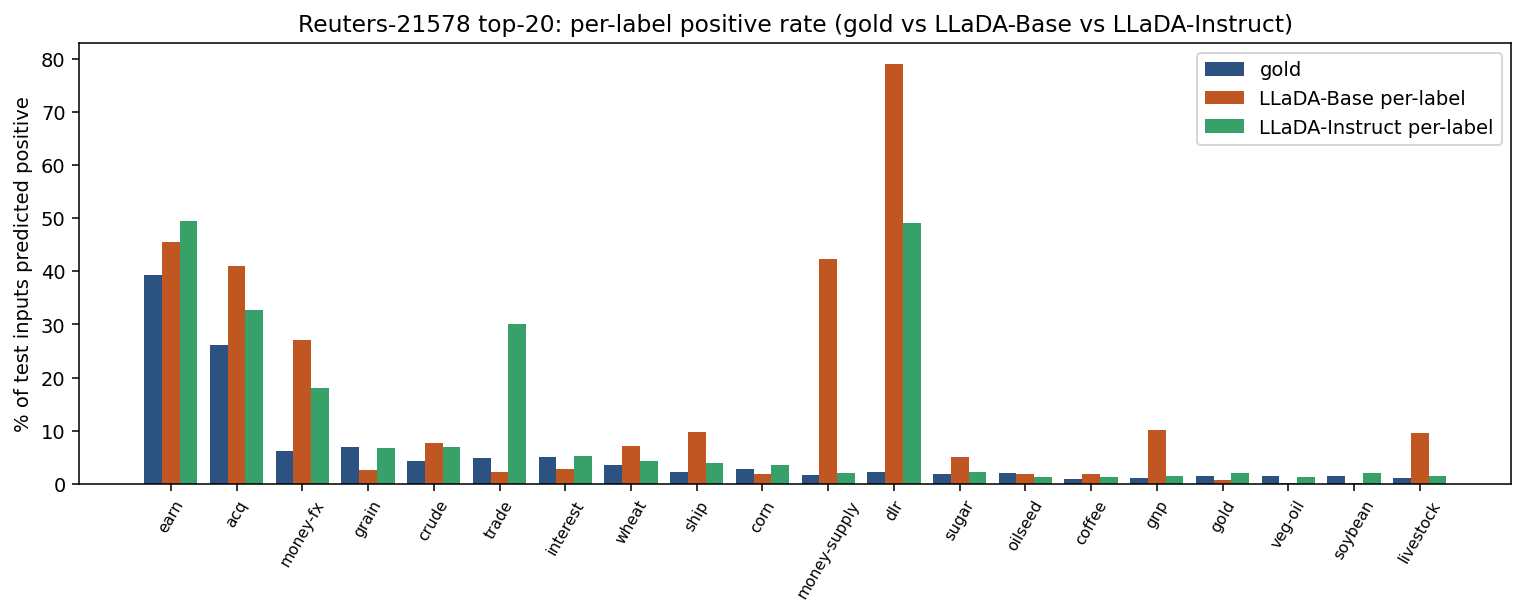}
\caption{Per-label positive prediction rate on Reuters-21578 top-20: gold labels (blue), LLaDA-Base per-label entailment (orange), LLaDA-Instruct (green). Base over-predicts \texttt{earn} and under-predicts most rare classes; Instruct tracks the gold rates much more closely. Mean absolute deviation from gold: $10.45\%$ (Base) vs $5.54\%$ (Instruct). Note that this is a per-label calibration/bias phenomenon, not the slot-position artefact of the all-masked unary stage (per-label scoring is permutation-invariant in the label ordering by \cref{thm:perm}).}
\label{fig:reuters_base_vs_instruct}
\end{figure}

The Base checkpoint over-predicts \texttt{earn}, \texttt{acq}, \texttt{money-fx}, \texttt{money-supply}, \texttt{oil}, \texttt{gnp}, and \texttt{livestock}. It also under-predicts \texttt{grain}, \texttt{trade}, \texttt{interest}, and several rare commodity labels. Instruct moves most bars toward the gold distribution and reduces the mean absolute rate error from $10.45\%$ to $5.54\%$. Its largest remaining mismatch is \texttt{trade}, where the positive rate is still much too high. These marginal rates do not measure example-level correctness, but they explain why macro-F1 improves when the Instruct checkpoint replaces Base.

\clearpage
\section{Accuracy--latency Pareto front.}
\label{app:pareto}

\Cref{fig:pareto} plots measured micro-F1 against per-example latency for the evaluated method and dataset pairs. Points toward the upper left are preferable because they combine higher accuracy with lower latency. Repeated method names correspond to different datasets, so the figure is a system-level summary rather than a within-dataset ranking.

\begin{figure}[H]
\centering
\includegraphics[width=\linewidth]{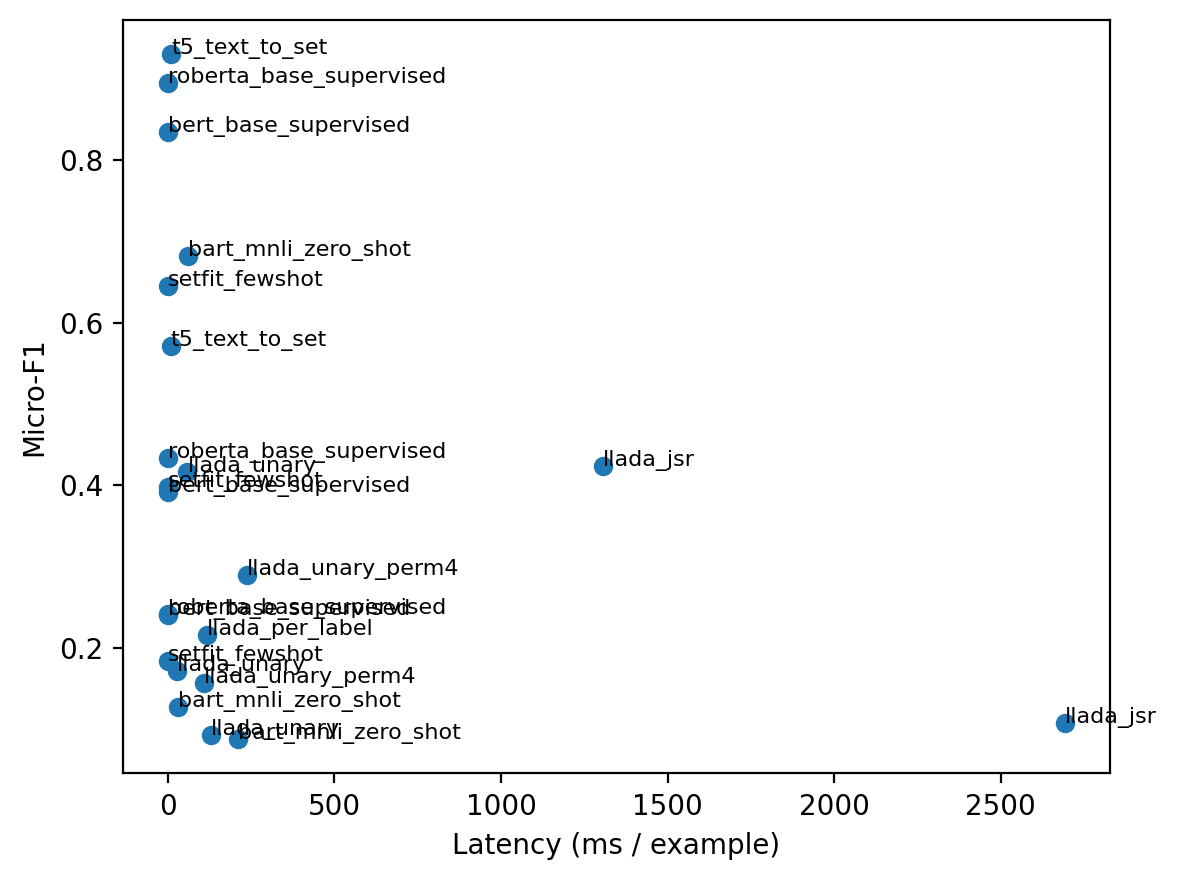}
\caption{Accuracy--latency Pareto front. Supervised methods cluster in the high-accuracy / sub-millisecond corner. BART-MNLI / SetFit / LLaDA-unary span $30$--$200$\,ms. LLaDA + JSR is the most expensive at $\sim$2000\,ms per example.}
\label{fig:pareto}
\end{figure}

The supervised encoders and T5 occupy the low-latency edge and reach the highest micro-F1 values, but they require full task-specific training. SetFit and BART-MNLI are slower than the supervised encoders yet remain below roughly 200\,ms per example in these runs. Per-label LLaDA and permutation-averaged unary move farther right because they require many masked-token queries. JSR is isolated at roughly 1.3 to 2.7 seconds per example and does not recover enough accuracy to justify that cost. This position agrees with the negative sweep results in \cref{app:jsr_sweeps}.

\section{Verbatim prompt templates.}
\label{app:prompts}

\paragraph{All-masked unary (USA).}
\begin{verbatim}
Decide whether each candidate label
applies. Use one token per label,
in order.

Document:
<doc>

Labels:
- <label_1>
- <label_2>
- ...
- <label_m>

Answers:
[MASK]; [MASK]; ...; [MASK]
\end{verbatim}

\paragraph{Per-label entailment.}
\begin{verbatim}
Document:
<doc>

Question: Does this document
express <label>?
Answer: [MASK]
\end{verbatim}

\paragraph{Verbalizers.}
LLaDA: $v^+ = $ \texttt{ yes}, $v^- = $ \texttt{ no} (with leading space). Both verify as single tokens at backbone load time.
Dream-7B (Qwen tokenizer): same verbalizers verify as single tokens.

\section{MDLM and BD3-LMs compatibility note.}
\label{app:mdlm}

The original MDLM release imports \texttt{flash\_attn} without an SDPA fallback, which prevents loading on our RTX 5090 (sm\_120) setup without code modification. The successor BD3-LMs provides an SDPA backend, but its forward interface requires a self-conditioning input that contains both the noisy state and a clean $x_0$ state, which our zero-shot adapter does not supply. We therefore leave MDLM and BD3-LMs evaluation to future work and use Dream-7B (\cref{app:dream}) as the second-backbone replication in this paper.

\section{Ethical considerations.}
\label{app:ethics}

The proposed method is training-free, but that convenience does not imply harmlessness. Multi-label classifiers are deployed in legal triage, content moderation, safety monitoring, and affective computing. Several risks: (i) pretrained diffusion LMs inherit biases from their pretraining corpora; (ii) the answer-slot design may create a false sense of transparency; (iii) prompt wording materially influences predictions; (iv) legal and emotion datasets contain sensitive text. The method should be positioned as an assistive tool with human oversight, subgroup analysis, calibration auditing, and dataset-specific harm review.

\section{JSR sweep tables (negative results).}
\label{app:jsr_sweeps}

The two tables below give the JSR sweep numbers cited from the body. They show lower F1 after local-conditional updates in the evaluated settings and seeds, so we retain $S=0$ as the default operating point.

\begin{table}[H]
\centering
\small
\renewcommand{\arraystretch}{1.08}
\caption{JSR sweeps from the all-masked unary (biased) seed on GoEmotions and Reuters-21578 (LLaDA-Base, $K=2$ MC mask contexts per local query). Performance degrades monotonically in $S$ from this seed.}
\label{tab:sweeps}
\begin{tabular}{lcccc}
\toprule
Dataset & $S=0$ & $S=1$ & $S=2$ & $S=3$ \\
\midrule
GoEmotions micro & 17.32 & 11.49 & 10.99 & 8.89 \\
GoEmotions macro & 8.91 & 5.83 & 5.61 & 4.72 \\
Reuters micro & 41.62 & 36.41 & 34.18 & 32.55 \\
Reuters macro & 10.94 & 9.73 & 9.20 & 8.78 \\
\bottomrule
\end{tabular}
\end{table}

Starting from the all-masked seed, every added sweep lowers both metrics on both datasets. From $S=0$ to $S=3$, GoEmotions loses 8.43 micro-F1 and 4.19 macro-F1 points; Reuters loses 9.07 and 2.16 points. The monotone decline in this table is an empirical pattern, distinct from the analytical counter-example in \cref{app:jsr_counterex}, but both point to the same failure mode: a locally preferred coordinate update need not improve the set-level prediction.

\begin{table}[H]
\centering
\small
\renewcommand{\arraystretch}{1.08}
\caption{JSR sweeps from the unbiased per-label seed on GoEmotions and Reuters (LLaDA-Instruct, $K=2$). Macro-F1 declines mildly on Reuters, and micro-F1 declines on both datasets under these settings.}
\label{tab:jsr_perlabel}
\begin{tabular}{lccc}
\toprule
Dataset / metric & $S=0$ & $S=1$ & $S=2$ \\
\midrule
GoEmotions micro & 26.79 & 20.83 & 19.78 \\
GoEmotions macro & 22.41 & 18.07 & 17.11 \\
Reuters micro & 60.70 & 47.86 & 47.52 \\
Reuters macro & 63.78 & 61.95 & 60.82 \\
\bottomrule
\end{tabular}
\end{table}

The per-label seed removes the positional bias but does not make local-JSR beneficial. Two sweeps reduce GoEmotions by 7.01 micro-F1 and 5.30 macro-F1 points. Reuters loses 13.18 micro-F1 points and 2.96 macro-F1 points. Since the degradation appears from both the biased and unbiased starting vectors, the problem lies in the local refinement objective rather than only in the quality of the seed.

\section{Preliminary results with LLaDA2.0-mini (16B MoE, CPU).}
\label{app:llada2mini}

As a preliminary scaling experiment, we evaluate the recently released \path|inclusionAI/LLaDA2.0-mini| checkpoint (16.26~B parameters, mixture-of-experts architecture with 1.4~B active parameters) using the same per-label entailment scoring recipe from the body.
Because the model does not fit on our RTX~5090 (32~GB), we run inference entirely on CPU in bfloat16 with int4 weight-only quantisation via \texttt{optimum-quanto}, which we validated to preserve every binary per-label decision relative to bf16 on a 6-pair smoke test while running ${\sim}2.4{\times}$ faster on long prompts.
Calibration uses 200 validation examples; test slices are 200 examples per dataset.
Results are in \cref{tab:llada2mini}.

\begin{table}[H]
\centering
\small
\renewcommand{\arraystretch}{1.08}
\caption{LLaDA2.0-mini (16B MoE, int4 on CPU) per-label entailment results on GoEmotions and Reuters-21578 (200 test examples each, single seed). LLaDA-8B-Instruct (GPU, full test slice) is shown for reference. The settings differ in checkpoint training, quantisation, and slice size, so the table is a preliminary comparison rather than a controlled scaling study.}
\label{tab:llada2mini}
\begin{tabular}{llrr}
\toprule
Model / template & Dataset & micro & macro \\
\midrule
LLaDA2.0-mini int4, default & GoEmotions & 25.3 & 26.7 \\
LLaDA2.0-mini int4, default & Reuters & 56.8 & 31.2 \\
LLaDA2.0-mini int4, ``main topic'' & Reuters & 67.1 & 38.4 \\
\midrule
LLaDA-8B-I (GPU ref), default & GoEmotions & 26.6 & 22.4 \\
LLaDA-8B-I (GPU ref), default & Reuters & 60.6 & 67.2 \\
LLaDA-8B-I (GPU ref), ``main topic'' & Reuters & 80.5 & 68.8 \\
\bottomrule
\end{tabular}
\end{table}

\noindent\textbf{Discussion.}
On GoEmotions, LLaDA2.0-mini has similar micro-F1 to LLaDA-8B-Instruct ($25.3$ vs $26.6$) and higher macro-F1 ($26.7$ vs $22.4$) in this small preliminary comparison. Architecture, quantisation, test-slice size, and checkpoint training differ, so the result cannot isolate an MoE effect.
On Reuters, the ``main topic'' template lifts micro from $56.8$ to $67.1$ ($+10.3$) and macro from $31.2$ to $38.4$ ($+7.2$), showing prompt sensitivity for this larger checkpoint on the 200-example slice.
Reuters macro ($38.4$) remains below LLaDA-8B-Instruct ($67.2$/$68.8$). Possible contributors include checkpoint training, calibration-slice size, and int4 quantisation; this experiment does not separate them.
An Instruct-tuned LLaDA2.0-mini, when released, would be the natural next experiment.
CPU inference is slow (${\sim}67$\,s / example on Reuters with the ``main topic'' template at int4, 20 labels) but runs in parallel with GPU workloads and requires no GPU memory.

\section{Reproducibility checklist.}
\label{app:repro}

\begin{itemize}[leftmargin=1.2em]
\item All code, configs, and saved score arrays released at \url{https://github.com/misterpawan/multilabel-classification-dllm-paper.git}.
\item Supervised checkpoints (BERT/RoBERTa/T5) on Hugging Face Hub.
\item All datasets used in our experiments are publicly accessible via standard sources such as Hugging Face Datasets or widely used public mirrors.
\item Single RTX 5090 (32 GB), bf16 inference, PyTorch 2.11 + cu128.
\item Three-seed reproducibility on headline cells in Section~\ref{app:multiseed}.
\item Every cell in the main table has a one-line CLI invocation.
\end{itemize}

\section{Full-gain JSR is monotone; tie-stable updates terminate.}
\label{app:fullgain}

\noindent\textbf{Proposition (full-gain JSR is monotone; tie-stable updates terminate).}\;
\textit{Fix an input $x$ and define the pseudo-likelihood surrogate
\[
\mathcal{PL}_\theta(y\mid x):=\sum_{i=1}^{m} \ell_{\theta,i}(y_i;x,y_{-i}),
\qquad y\in\{0,1\}^{m}.
\]
Consider the \emph{full-gain} coordinate update at step $t$:
\[
y_i^{(t+1)} \in \arg\max_{b\in\{0,1\}}
\mathcal{PL}_\theta\bigl(y_1^{(t+1)},\dots,y_{i-1}^{(t+1)}, b, y_{i+1}^{(t)},\dots,y_m^{(t)} \mid x\bigr),
\]
with the \emph{tie-stable} rule that if the current value $y_i$ is already an argmax, it is kept unchanged.
Then: (i) each coordinate update is monotone non-decreasing in $\mathcal{PL}_\theta$; (ii) every full sweep is monotone non-decreasing in $\mathcal{PL}_\theta$; (iii) because the state space $\{0,1\}^m$ is finite and tie updates do not move when no strict improvement is available, repeated sweeps terminate after finitely many coordinate changes at a coordinate-wise local optimum of $\mathcal{PL}_\theta$.
The local update Eq.~\ref{eq:jsr} does \textbf{not} have properties (i)--(iii): changing $y_i$ also alters every $\ell_{\theta,j}(y_j;x,y_{-j})$ for $j\ne i$ (because $y_i$ sits inside $y_{-j}$), so maximising the single term $\ell_{\theta,i}$ need not increase the sum (see \cref{app:jsr_counterex}).}

\noindent\textit{Proof.}\; Fix a step of the coordinate update for index $i$. By construction,
\[
y_i^{(t+1)} \in \arg\max_{b\in\{0,1\}}
\mathcal{PL}_\theta\bigl(y_1^{(t+1)},\dots,y_{i-1}^{(t+1)}, b, y_{i+1}^{(t)},\dots,y_m^{(t)} \mid x\bigr).
\]
Therefore $\mathcal{PL}_\theta$ after the $i$-th update is at least $\mathcal{PL}_\theta$ before it, which proves monotonicity of each coordinate update; summing across the coordinates in a sweep proves monotonicity of each full sweep.

For termination, note that the process can only change state when either (a) $\mathcal{PL}_\theta$ strictly increases, or (b) the surrogate stays unchanged but the chosen coordinate value changes to another maximiser.
Under the stated tie-stable rule, case (b) cannot occur when the current coordinate value is already optimal, because ties keep the current value unchanged.
Hence every actual state change strictly increases $\mathcal{PL}_\theta$.
Since $\{0,1\}^m$ is finite, only finitely many strict increases are possible, so the process terminates after finitely many coordinate changes.
At termination, no single coordinate can be changed to improve $\mathcal{PL}_\theta$ while holding the others fixed (otherwise the update would perform that strict improvement); the terminal point is therefore a coordinate-wise local optimum. \qed

\subsection{Counter-example: local-JSR can decrease
\texorpdfstring{$\mathcal{PL}_\theta$}{PL theta}.}
\label{app:jsr_counterex}
A two-label instance suffices.
Let $|\Lambda|=2$ with local scores
\begin{align*}
\ell_{\theta,1}(1; y_2) &= -0.1, & \ell_{\theta,1}(0; y_2) &= -1.0 \text{ for all } y_2,\\
\ell_{\theta,2}(1; y_1{=}0) &= -0.1, & \ell_{\theta,2}(1; y_1{=}1) &= -10.0,\\
\ell_{\theta,2}(0; \cdot) &= -0.2.
\end{align*}
Starting at $y=(0,1)$, the local update for coordinate $1$ (Eq.~\ref{eq:jsr}) picks $y_1=1$ because $-0.1 > -1.0$.
But
\[
\mathcal{PL}_\theta(0,1) = -1.0 + (-0.1) = -1.1,\quad \mathcal{PL}_\theta(1,1) = -0.1 + (-10.0) = -10.1,
\]
so the surrogate \emph{decreases} by $9.0$.
The full-gain update would instead pick $b\in\arg\max_b \mathcal{PL}_\theta(b, 1)$, namely $b=0$, leaving $y_1$ unchanged and preserving $\mathcal{PL}_\theta=-1.1$.
This example shows that local-JSR can decrease the surrogate $\mathcal{PL}_\theta$ and is therefore not coordinate ascent on $\mathcal{PL}_\theta$; this is the underlying reason the empirical local-JSR curves in \cref{tab:sweeps,tab:jsr_perlabel} move in the wrong direction.

\end{document}